\documentclass[reqno]{amsart}

 \usepackage[foot]{amsaddr}
\usepackage{amsthm}
\usepackage{enumerate}
\usepackage{color}
\usepackage{float}
\usepackage{bbm}
\usepackage{mathrsfs}
\usepackage{amsmath}
\usepackage{comment}
\usepackage{listings}
\usepackage[dvipsnames]{xcolor}
\usepackage{fourier}
\usepackage[ruled,noend]{algorithm2e}
\SetKwInOut{Input}{Input}
\SetKwInOut{Output}{Output}
\SetKw{Return}{Return}
\usepackage{hyperref,cleveref}
\usepackage[colorinlistoftodos]{todonotes}
\hypersetup{
	colorlinks,
	citecolor=blue,
	filecolor=blue,
	linkcolor=blue,
	urlcolor=blue
}
\allowdisplaybreaks

\newtheorem{theorem}{Theorem}[section]

\newtheorem{lemma}[theorem]{Lemma}

\newtheorem{corollary}[theorem]{Corollary}
\newtheorem{definition}[theorem]{Definition}
\newtheorem{assumption}[theorem]{Assumption}

\theoremstyle{definition}

\newtheorem{remark}[theorem]{Remark}

\newcommand{\HH}{\mathcal{H}}

\newcommand{\R}{\mathbb{R}}

\newcommand{\E}{\mathbb{E}}
\newcommand{\Prob}{\mathbb{P}}

\newcommand{\diag}{\textnormal{diag}}
\newcommand{\rank}{\textnormal{rank}}
\newcommand{\WLOT}{W_{2,\sigma}^{\operatorname{LOT}}}
\newcommand{\WLOTest}{\widehat{W}_{2,\sigma}^{\operatorname{LOT}}}
\newcommand{\fro}{{\mathrm{F}}} 
\newcommand{\LL}{\mathrm{L}}
\def\mc{\mathcal}
\def\ac{\textnormal{ac}}

\newcommand{\mfixAdd}{\mu}

\newcommand{\wspaceR}{W_2(\mathbb{R}^n)}
\newcommand{\supp}{\operatorname{supp}}

\newcommand{\argmin}{\text{argmin}}

\newcommand{\eps}{\varepsilon}

\newcommand{\corr}[1]{\textcolor{black}{#1}}

\title{Linearized Wasserstein dimensionality reduction with approximation guarantees}

\begin{document}

\author{Alexander Cloninger$^{1,2}$}\address{$^1$Department of Mathematics, University of California, San Diego, CA}
\address{$^2$Halicio\u{g}lu Data Science Institute, University of California, San Diego, CA}\email{acloninger@ucsd.edu}

\author{Keaton Hamm$^{3,4}$}
\address{$^3$Department of Mathematics, University of Texas at Arlington, Arlington, TX}
\address{$^4$Division of Data Science, University of Texas at Arlington, Arlington, TX}
\email{keaton.hamm@uta.edu}

\author{Varun Khurana$^5$}\address{$^5$Department of Applied Mathematics, Brown University, Providence, RI 02906}\email{varun\_khurana@brown.edu}

\author{Caroline Moosm\"{u}ller$^6$}\address{$^6$Department of Mathematics, University of North Carolina at Chapel Hill, NC}\email{cmoosm@unc.edu}


\begin{abstract}
We introduce LOT Wassmap, a computationally feasible algorithm to uncover low-dimensional structures in the Wasserstein space. The algorithm is motivated by the observation that many datasets are naturally interpreted as probability measures rather than points in $\R^n$, and that finding low-dimensional descriptions of such datasets requires manifold learning algorithms in the Wasserstein space. Most available algorithms are based on computing the pairwise Wasserstein distance matrix, which can be computationally challenging for large datasets in high dimensions. Our algorithm leverages approximation schemes such as Sinkhorn distances and linearized optimal transport to speed-up computations, and in particular, avoids computing a pairwise distance matrix. We provide guarantees on the embedding quality under such approximations, including when explicit descriptions of the probability measures are not available and one must deal with finite samples instead. Experiments demonstrate that LOT Wassmap attains correct embeddings and that the quality improves with increased sample size. We also show how LOT Wassmap significantly reduces the computational cost when compared to algorithms that depend on pairwise distance computations.
\end{abstract}

\maketitle

\setcounter{tocdepth}{1}
\tableofcontents

\section{Introduction}

A classical problem in analyzing large volume, high-dimensional datasets is to develop efficient algorithms that classify points based on a similarity measure, or based on a subset of preclassified training data points. Even when data points lie in high-dimensional Euclidean space, they can often be approximated by low-dimensional structures, such as subspaces or submanifolds. This observation has led to significant advances in the field, mostly through the development of \emph{manifold learning algorithms}, which produce a low-dimensional representation of a given dataset; see for example \cite{belkin2003laplacian,coifman2006diffusion,maaten2008visualizing,tenenbaum2000global}. In many of these frameworks, the data points are assumed to be sampled from a low-dimensional Riemannian manifold embedded in Euclidean space, and approximately preserve intrinsic properties such as geodesic distances.

In many applications however, data points are more naturally interpreted as distributions $\{\mu_i\}_{i=1}^{N}$ over $\mathbb{R}^n$, or finite samples $X_i = \{x_j^{(i)}\}_{j=1}^{N_i}$ with $x_j^{(i)} \sim \mu_i$. Examples include imaging data \cite{rubner2000earth}, text documents (the bag-of-word model uses word count within a text as features, creating a histogram for each document \cite{zhang2010understanding}), and gene expression data, which can be interpreted as a distribution over a gene network \cite{Chen:2017aa,mathews18}. In this setting, a Euclidean embedding space with Euclidean distances locally approximating the intrinsic distance of the data may not be geometrically meaningful, and datasets are better modeled as probability measures in the \emph{Wasserstein space} \cite{villani2008optimal}.

We assume that our data points $\{\mu_i\}_{i=1}^N$ belong to the quadratic Wasserstein space $W_2(\mathbb{R}^n)$ of probability measures with finite second moment, equipped with the Wasserstein distance
\begin{equation}\label{eq:W2-metric-intro}
    W_2(\mu,\nu) := \inf_{\pi\in\Gamma(\mu,\nu)} \left(\int_{\R^{2n}}\|x-y\|^2d\pi(x,y)\right)^\frac12,
\end{equation}
where $\mathcal{P}(\R^{2n})$ is the set of all probability measures over $\R^{2n}$ and $\Gamma(\mu,\nu):=\{\gamma\in\mathcal{P}(\R^{2n}): \gamma(A\times \R^n) = \mu(A),\; \gamma(\R^n\times A)=\nu(A) \textnormal{ for all } A\subset\R^n\}$ is the set of all joint probability measures with marginals $\mu$ and $\nu$.
Under regularity assumptions on $\mu$, the optimal coupling $\pi$ has the form $\pi=(\operatorname{id},T)_{\sharp}\mu$, where $T \in \LL^2(\mathbb{R}^n,\mu)$ is the ``optimal transport map'' \cite{brenier1991polar,villani2008optimal}.

The Wasserstein space and optimal transport have gained popularity in the machine learning community, as they are based on a solid theoretical foundation \cite{villani2008optimal} (for example, \eqref{eq:W2-metric-intro} is a metric), while providing a versatile framework for applications (for example, as a cost function for generative models \cite{arjovsky2017wasserstein}, in semi-supervised learning \cite{solomon2014wasserstein}, and in pattern detection for neuronal data \cite{mishne2016hierarchical}). 


In this paper, we are interested in uncovering low-dimensional submanifolds in the Wasserstein space in a \emph{computationally feasible} manner as well as analyzing the quality of the embedding. 
To this end, we follow the idea of \cite{hamm2022wassmap,wang2010optimal}, which introduces the \emph{Wassmap} algorithm (see \Cref{sec:Wassmap} for more details), a version of the Multidimensional Scaling algorithm (MDS) \cite{mardia1979multivariate} (see \Cref{ALG:MDS}), or more generally, the Isomap algorithm \cite{tenenbaum2000global}.

A central part of manifold learning algorithms like MDS or Isomap relies on the computation of the pairwise Euclidean distances. Wassmap uses the pairwise Wasserstein distance matrix instead, which leads to $O(N^2)$ Wasserstein distance computations, each of which is of the order $O(n^3\log(n))$ if one uses interior point methods to solve the linear program \eqref{eq:W2-metric-intro}. If both $N$ and $n$ are large, computing all pairwise distances becomes infeasible. To deal with this issue, approximations of the Wasserstein distance can be considered. In this paper, we are interested in \emph{entropic regularized} distances (Sinkhorn distances) \cite{altschuler2017nearlinear,cuturi2013sinkhorn}, which deal with the computational issue involving $n$, and in \emph{linearized optimal transport} (LOT) \cite{gigli-2011,wang2010optimal}, to reduce the computational cost in $N$.

Our results are twofold: 
\begin{enumerate}
    \item \label{it:guarantees}\textbf{Approximation guarantees}:
    \begin{itemize}
        \item We provide bounds on the embedding quality of the Multidimensional Scaling algorithm (MDS) \cite{mardia1979multivariate} (see \Cref{ALG:MDS}) applied to a dataset in the Wasserstein space, where the pairwise Wasserstein distances are only available up to an error $\tau$.
        \item We study the size of $\tau$ in common approximation schemes such as entropic regularization and linearized approximations, \emph{and} when explicit descriptions of the data points $\mu_i, i=1,\ldots,N$ are not available, and one must deal with finite samples instead.
    \end{itemize}
     \item \textbf{Efficient algorithm (LOT Wassmap)}: We provide an algorithm, ``LOT Wassmap'', inspired by the Wassmap algorithm of \cite{hamm2022wassmap}. It essentially uses linearized Wasserstein distance approximations through LOT in the Multidimensional Scaling algorithm, leveraging our approximation guarantees from (\ref{it:guarantees}). However, we \emph{do not} compute the LOT-Wasserstein distance matrix and feed it into MDS, but instead compute the truncated SVD of centered transport maps. This is the same in theory, but computationally more efficient.
\end{enumerate}

\subsection{Previous work}
The idea of replacing pairwise Euclidean distances with pairwise Wasserstein distances in common manifold learning algorithms has been explored in many settings; for example in \cite{zelesko2020earthmoverDiffusion} to study shape spaces of proteins, in \cite{mathews18,Chen:2017aa} to analyze gene expression data, and in \cite{wang2010optimal} for cancer detection. 

Theoretical results on the reconstruction of certain submanifolds in $W_2(\R^n)$ through the MDS algorithm using pairwise Wasserstein distances are presented in \cite{hamm2022wassmap}. The associated algorithm, Wassmap, is the basis for our LOT Wassmap algorithm.

Related to the idea of uncovering submanifolds in the Wasserstein space is ``Wasserstein dictionary learning'' as discussed in \cite{mueller2022geometric,werenski2022measure}. The authors propose to represent complex data in the Wasserstein space as Wasserstein barycenters of a dictionary.

\subsection{Approximation guarantees}
Using approximations of the Wasserstein distance in manifold learning algorithms such as MDS may change the embedding quality, and our main result provides theoretical bounds on the error:
\begin{theorem}[Informal version of \Cref{THM:Meta}]\label{intro:main_result}
    Assume that data points $\{\mu_i\}_{i=1}^N$ are $\tau_1-$close to a $d$-dimensional submanifold $\mathcal{W}$ in the Wasserstein space, which is isometric to a subset $\Xi$ of Euclidean space $\mathbb{R}^d$. Furthermore assume that we only have access to approximations $\lambda_{ij}$ of the pairwise distances $W_2(\mu_i,\mu_j)$, and that the approximation error is $\tau_2$.

    Then, under some technical assumptions, the Multidimensional Scaling algorithm using distances $\lambda_{ij}$ as input recovers data points $\{z_i\}_{i=1}^N\subset \mathbb{R}^d$, which are $C_{N,\mathcal{W}}(\tau_1+\tau_2)$-close to $\Xi$ up to rigid transformations.
\end{theorem}
Some remarks on this result:
\begin{itemize}
    \item The first source of error, $\tau_1$, depends on how close the data points are to the submanifold $\mathcal{W}$ isometric to a subspace of $\mathbb{R}^d$, which is completely determined by the dataset.
    \item The second source of error, $\tau_2$, depends on the approximation scheme used, and can be made arbitrarily small with sufficient computational time or good choice of parameters.
\end{itemize}
A significant part of this paper is dedicated to providing bounds for $\tau_2$, when common approximation schemes for $W_2(\mu_i,\mu_j)$ are used, and when $\{\mu_i\}_{i=1}^N$ are only available through samples, i.e.\ when $\mu_i \approx \widehat{\mu}_i= \frac{1}{m_i}\sum_{j=1}^{m_i}\delta_{Y_{j}^{(i)}}$ with $Y_{j}^{(i)} \sim \mu_i$ i.i.d.
In particular, we introduce \emph{empirical linearized Wasserstein-2 distance}, $\WLOTest$, which uses two approximation schemes:
\begin{enumerate}[(a)]
    \item \emph{Entropic regularized formulation}: A very successful approximation framework for efficient Wasserstein distance computation is the entropic regularized formulation of \eqref{eq:W2-metric-intro}, which depends on a parameter $\beta$, and leads to \emph{Sinkhorn distances} \cite{cuturi2013sinkhorn}:
\begin{equation}\label{EQ:SinkProb-intro}
    \underset{\pi \in \Gamma(\mu,\nu)}\min \int_{\R^{2n}} \frac12 \Vert x - y \Vert^2 d\pi(x,y) + \beta D_{\textnormal{KL}}(\pi \Vert \mu \otimes \nu ),
\end{equation}
where $D_{\textnormal{KL}}$ is the Kullback--Leibler divergence of measures \cite{joyce2011kullback}. This formulation leads to a unique solution (in contrast to \eqref{eq:W2-metric-intro}), and to a significant computational speed-up in $n$, achieving $O(n^2\log(n))$ through matrix scaling algorithms (Sinkhorn's algorithm) \cite{altschuler2017nearlinear,cuturi2013sinkhorn}.
\item \emph{Linearized Wasserstein distances}: Linearized optimal transport (LOT) \cite{gigli-2011,wang2010optimal} approximates Wasserstein distances by linear $\LL^2-$distances in the tangent space at a chosen reference measure $\sigma$:
\begin{equation}\label{intro:LOT}
     \WLOT(\mu,\nu) := \left(\int_{\R^n}\|T_{\sigma}^{\mu}(x)-T_{\sigma}^{\nu}(x)\|^2\,d{\sigma}(x)\right)^{1/2},
\end{equation}
where $T_{\sigma}^{\mu}$ denotes the optimal transport map from $\sigma$ to $\mu$ (either computed through \eqref{eq:W2-metric-intro} or \eqref{EQ:SinkProb-intro}, and using barycentric projections to make a transport plan into a transport map).
Instead of computing all pairwise optimal transport maps, in this framework, one computes $T_{\sigma}^{\mu_i}$ from $\sigma$ to $\mu_i$, and approximates pairwise maps between $\mu_i$ and $\mu_j$ as a composition of $T_{\sigma}^{\mu_i}$ and $T_{\sigma}^{\mu_j}$, reducing the computation in $N$ to $O(N)$. This framework has been successfully applied signal and image classification tasks \cite{park18,wei13}, such as visualizing phenotypic differences between types of cells \cite{basu14}. There furthermore exist error bounds for $\WLOT$ \cite{berman20,merigot2021,gigli-2011,khurana2022supervised,merigot20,moosmuller2020linear}.
\end{enumerate}
With these approximation schemes at hand, we define the \emph{empirical linearized Wasserstein-2 distance}:
\begin{equation}\label{eq:W2-LOT-empir-intro}
    \WLOTest(\widehat{\mu},\widehat{\nu}) :=\left(\frac{1}{m}\sum_{j=1}^m \|T_\sigma^{\widehat{\mu}}(X_j) - T_\sigma^{\widehat{\nu}}(X_j)\|^2\right)^{1/2},
\end{equation}
where $X_j \sim \sigma$ i.i.d.\ and the transport maps are either computed by \eqref{eq:W2-metric-intro} or \eqref{EQ:SinkProb-intro} (and with barycentric projections, if necessary).

We provide values for $\tau_2$ as in \Cref{intro:main_result}, by bounding $|W_2(\mu,\nu)^2-\WLOTest(\widehat{\mu},\widehat{\nu})^2|$, using either a linear program or Sinkhorn iterations to compute the transport plans. These bounds are derived by combining the following results:
\begin{itemize}
    \item Estimation of optimal transport maps with plug-in estimators, i.e.\ bounds on $\|T_{\mu}^{\widehat{\nu}}-T_{\mu}^{\nu}\|_{\mu}$, which are provided by \cite{deb2021rates} for the linear program case, and by \cite{pooladian2021} in the regularized case. Both \cite{deb2021rates} and \cite{pooladian2021} assume compactly supported $\mu$ and $\nu$, while we are able to relax the compact support assumption on the target measure, as long as it can be approximated by compactly supported measures.
    \item Approximation results for $\WLOT$, which are provided in \cite{khurana2022supervised,moosmuller2020linear}, and are based on the idea that $\mu_i$ are generated by almost compatible functions $\HH$ applied to a fixed generator $\mu$. We also strengthen some of the approximation results in \cite{khurana2022supervised,moosmuller2020linear}.
\end{itemize}

\subsection{Efficient algorithm: LOT Wassmap}
The Wassmap algorithm of \cite{hamm2022wassmap} requires computing the pairwise Wasserstein distance matrix $W_2(\mu_i,\mu_j)$, $i,j=1,\ldots,N$, which leads to $O(N^2)$ expensive computations. We introduce \emph{LOT Wassmap} (see \Cref{alg:lotWass}), which uses LOT distances \eqref{intro:LOT} to linearly approximate $W_2(\mu_i,\mu_j)$ (since the input of our algorithm are empirical samples $\widehat{\mu}_i$, we actually use the empirical linearized Wasserstein-2 distance \eqref{eq:W2-LOT-empir-intro}). This results in only $O(N)$ optimal transport computations.

However, in practice, we avoid computing the pairwise LOT distance matrix. Instead, we compute the truncated SVD of the centered transport maps, which is computationally more efficient. We show that in theory this produces a result equivalent to \Cref{intro:main_result}:
\begin{corollary}[Informal version of \Cref{COR:LOTWassGuarantee}]\label{COR:LOTWassGuarantee-intro}
Assume that data points $\{\mu_i\}_{i=1}^N$ are $\tau_1-$close to a $d$-dimensional submanifold $\mathcal{W}$ in the Wasserstein space, which is isometric to a subset $\Xi$ of Euclidean space $\mathbb{R}^d$. Choose a reference measure $\sigma$ and compute all transport maps $T_{\sigma}^{\mu_i}$ (either with a linear program \eqref{eq:W2-metric-intro} or with Sinkhorn approximations \eqref{EQ:SinkProb-intro}, and with barycentric projections, if necessary). Let $\tau_2$ be the error between the empirical linearized Wasserstein-2 distance $\WLOTest(\widehat{\mu}_i,\widehat{\mu}_j)$ of \eqref{eq:W2-LOT-empir-intro} and the actual Wasserstein-2 distance $W_2(\mu_i,\mu_j)$.

Then, under some technical assumptions, the truncated SVD of the centered transport maps $T_{\sigma}^{\mu_i}$ (column-stacked) produces data points $\{z_i\}_{i=1}^N \subset \R^d$, which are $C_{N,\mathcal{W}}(\tau_1+\tau_2)$-close to $\Xi$ up to rigid transformations.
\end{corollary}
We note that \Cref{COR:LOTWassGuarantee-intro} is a corollary of \Cref{intro:main_result} and that the technical assumptions and constants are the same in both results.

In \Cref{sec:experiments}, we provide experiments demonstrating that LOT Wassmap does attain correct embeddings given finite samples without explicitly computing the pairwise LOT distance matrix. In particular, we show that the embedding quality improves with increased sample size and that LOT Wassmap significantly reduces the computational cost when compared to Wassmap.

\subsection{Organization of the paper}
This paper is organized as follows: We start by introducing important notation and background in \Cref{sec:notation}. This includes discussion of the MDS and Wassmap algorithms, (linearized) optimal transport, and plug-in estimators. \Cref{sec:maintheorem} introduces the LOT Wassmap algorithm and provides the main results. \Cref{sec:compact_support,sec:noncompact-bounds} provide approximation guarantees for $\WLOTest(\widehat{\mu},\widehat{\nu})$ for compactly and non-compactly supported target measures, respectively. The approximation guarantees come with many technical assumptions, and \Cref{SEC:CondCpt,SEC:CondNonCpt} are dedicated to discussing settings in which these assumptions hold. The paper concludes with experiments in \Cref{sec:experiments}, which show the effectiveness of LOT Wassmap. Proofs are provided in \Cref{ap:helpertheo,ap:estimatorapprox,ap:noncompactProofs,ap:conditions}.

\section{Notation and Background}\label{sec:notation}

This paper has a significant amount of background and notation which is summarized categorically here. See \Cref{tab:notation} for an overview of notation used in the paper.


\begin{table}[h!]
\begin{tabular}{|c|p{0.6\linewidth}|c|}
\hline
\textbf{Notation} & \textbf{Definition} & \textbf{Reference}\\
\hline
$\Delta$ & Square Euclidean distance matrix & \Cref{ALG:MDS}\\
\hline
$\Lambda$ & Perturbed distance matrix & \Cref{COR:MDSPerturbation}\\ 
\hline
$X^\dagger$ & Moore--Penrose pseudoinverse of matrix $X$ & \Cref{sec:LinAlg}\\
\hline
$\mu$ & Template measure  & \Cref{sec:LOT-intro} \\
\hline
$\widehat{\mu}$ & Empirical measure approximating $\mu$ & \eqref{eq:estimator}\\
\hline
$\sigma$ & Reference measure for LOT & \Cref{sec:LOT-intro}\\
\hline
$\|\cdot\|_{\textnormal{S}_p}$ & Schatten $p$-norm & \Cref{sec:LinAlg}\\
\hline
$\|\cdot\|$ & Spectral norm of a matrix or Euclidean norm of a vector & \Cref{sec:LinAlg} \\
\hline
$\|\cdot\|_\fro$ & Frobenius norm of a matrix & \Cref{sec:LinAlg} \\
\hline
$\|\cdot\|_{\max}$ & (Entrywise) maximum norm of a matrix & \Cref{sec:LinAlg} \\
\hline
$\|\cdot\|_\mu$ & Norm on $\LL^2(\R^n,\mu)$ & \Cref{sec:OT-prelims}\\
\hline
$ n $ & Dimension of Euclidean space that probability measures are defined on & \Cref{sec:OT-prelims}\\
\hline
$\mathcal{P}(\R^n)$ & Probability measures on $\R^n$& \Cref{sec:OT-prelims}\\
\hline
$\mathcal{P}_{\textnormal{ac}}(\R^n)$ & Absolutely continuous probability measures on $\R^n$ & \Cref{sec:OT-prelims}\\
\hline
$W_2(\R^n)$ & Wasserstein-$2$ space over $\R^n$ & \Cref{sec:OT-prelims}\\
\hline
$W_2(\mu,\nu)$ & Wasserstein-$2$ distance between $\mu$ and $\nu$ & \eqref{eq:W2-metric}\\
\hline
\rule{0pt}{\normalbaselineskip}  $\WLOT(\mu,\nu)$ & Linearized Wasserstein-$2$ distance between $\mu$ and $\nu$, with $\sigma$ as reference & \eqref{eq:W2-LOT}\\
\hline
\rule{0pt}{\normalbaselineskip} $\WLOTest(\mu,\nu)$ & Empirical linearized Wasserstein-$2$ distance &  \eqref{eq:W2-LOT-empir-all}\\
\hline
\rule{0pt}{\normalbaselineskip}  $T_\sigma^\mu$ & Optimal transport (Monge) map from $\sigma$ to $\mu$ & \Cref{sec:OT-prelims}\\
\hline
$T_\sharp \mu$ & Pushforward of $\mu$ with respect to $T$ & \Cref{sec:OT-prelims}\\
\hline
\rule{0pt}{\normalbaselineskip} $T_\sigma^{\widehat{\mu}}$ & Barycentric projection of an optimal transport plan (Kantorovich potential) &  \eqref{baryEst}\\
\hline
$d$ & Embedding dimension of MDS & \Cref{SEC:MDS} \\
\hline
$k$ & Sample size that generates $\widehat{\mu}$ &  \eqref{eq:estimator} \\
\hline
$m$ & Sample size that generates $\widehat{\sigma}$ & \Cref{alg:lotWass} \\
\hline
$N$ & Number of data points & \Cref{alg:lotWass}\\ 
\hline
$\eps$ & Distance from compatibility & \Cref{epsCompatDef}\\
\hline
$\beta$ & Regularizer for Sinkhorn OT & \Cref{sec:Sinkhorn} \\
\hline
\end{tabular}
\caption{Overview of notation used in the paper.}
\label{tab:notation}
\end{table}


\subsection{Linear Algebra Preliminaries}\label{sec:LinAlg}

Given $A\in\R^{m\times n}$, its {\it Singular Value Decomposition} (SVD) is given by $A = U\Sigma V^\top$, where $U\in\R^{m\times m}$ and $V\in\R^{n\times n}$ are orthogonal matrices and $\Sigma\in\R^{m\times n}$ has non-zero entries along its main diagonal (singular values). The singular values are the square roots of the eigenvalues of $A^\top A$ and are taken in descending order $\sigma_1\geq\sigma_2\geq\dots\geq\sigma_{\min\{m,n\}}\geq0$. The truncated SVD of order $d$ of $A$ is $A_d = U_d\Sigma_d V_d^\top$ where $U_d$ and $V_d$ consist of the first $d$ columns of $U$ and $V$, respectively, and $\Sigma_d = \diag(\sigma_1,\dots,\sigma_d)\in\R^{d\times d}$.  The Moore--Penrose pseudoinverse of $A\in\R^{m\times n}$ is the $n\times m$ matrix denoted by $A^\dagger$ and defined by $A^\dagger = V\Sigma^\dagger U^\top$ where $\Sigma^\dagger$ is the $n\times m$ matrix with entries $\frac{1}{\sigma_1},\dots,\frac{1}{\sigma_{\min\{m,n\}}}$ along its main diagonal.

The Schatten $p$-norms ($1\leq p\leq\infty$) are a general class of unitarily invariant, submultiplicative norms on $\R^{m\times n}$ and are defined to be the $\ell^p$ norms of the vector of singular values: $\|A\|_{\textnormal{S}_p} := \|(\sigma_1,\dots,\sigma_{\min\{m,n\}})\|_{\ell_p}.$ The Frobenius norm, which is the Schatten $2$-norm is denoted by $\|\cdot\|_\fro$, and the spectral norm, which is the Schatten $\infty$-norm is denoted simply by $\|\cdot\|$. We also use $\|\cdot\|$ to denote the Euclidean norm of a vector.





\subsection{Multidimensional scaling}\label{SEC:MDS}

Let $\mathbf{1}$ be the all-ones vector in $\R^N$, and $J := I-\frac{1}{N}\mathbf{1}\mathbf{1}^\top$.  Then Multidimensional Scaling (MDS) is summarized in \Cref{alg:MDS}. For more details see \cite{mardia1979multivariate}.









\begin{algorithm}[h!] 
 \caption{Multidimensional Scaling (MDS) \cite{mardia1979multivariate}}\label{ALG:MDS}
\SetKwInput{Input}{Input}
\SetKwInOut{Output}{Output}
\SetKw{Return}{Return}
\Input{~Points $\{y_i\}_{i=1}^N\subset \R^D$; embedding dimension $d\ll D$.}

\Output{Low-dimensional embedding points $\{z_i\}_{i=1}^N\subset\R^d$}

Compute pairwise distance matrix $\Delta_{ij} = \|y_i-y_j\|^2$\\

$B = -\frac12 J\Delta J$\\

(Truncated SVD): $B_d=V_d\Sigma_d V_d^\top$\\

$z_i = (V_d\Sigma_d)(i,:),$ for $i=1,\dots,N$\\

\Return{$\{z_i\}_{i=1}^N$}
\label{alg:MDS}
\end{algorithm}

MDS produces an isometric embedding $\R^D\to\R^d$ if and only if the matrix $B$ is symmetric positive semi-definite with rank $d$, a result that goes back to Young and Householder \cite{young1938discussion}. In this case, the embedding points $\{z_i\}_{i=1}^N\subset\R^d$ satisfy $\|z_i-z_j\|=\|y_i-y_j\|$ and are unique up to rigid transformation.

\subsection{Optimal Transport Preliminaries}\label{sec:OT-prelims}

Let $\mc{P}(\R^n)$ be the space of all probability measures on $\R^n$, with $\mc{P}_{\ac}(\R^n)$ being the subset of all probability measures which are absolutely continuous with respect to the Lebesgue measure. Given $\mu\in\mc{P}_{\ac}(\R^n)$, we denote its probability density function by $f_\mu$. The quadratic Wasserstein space $W_2(\R^n)$ is the subset of $\mathcal{P}(\R^n)$ of measures with finite second moment $\int_{\R^n}\|x\|^2d\mu(x)<\infty$ equipped with the quadratic Wasserstein metric given by
\begin{equation}\label{eq:W2-metric}
    W_2(\mu,\nu) := \inf_{\pi\in\Gamma(\mu,\nu)} \left(\int_{\R^{2n}}\|x-y\|^2d\pi(x,y)\right)^\frac12,
\end{equation}
where $\Gamma(\mu,\nu):=\{\gamma\in\mathcal{P}(\R^{2n}): \gamma(A\times \R^n) = \mu(A),\; \gamma(\R^n\times A)=\nu(A) \textnormal{ for all } A\subset\R^n\}$ is the set of couplings, i.e., measures on the product space whose marginals are $\mu$ and $\nu$. 

In \cite{brenier1991polar}, Brenier showed that if $\mu$ is absolutely continuous with respect to the Lebesgue measure, the optimal coupling of \eqref{eq:W2-metric} takes the special form $\pi = (\operatorname{id},T_{\mu}^{\nu})_{\sharp}\mu$, where $\sharp$ is the pushforward operator ($S_{\sharp}\mu(A) =\mu(S^{-1}(A))$ for $A$ measurable) and $T_{\mu}^{\nu}\in \LL^2(\R^n,\mu)$ solves 
\begin{equation*}
    \min_{T:T_{\sharp}\mu=\nu}\int_{\R^n}\|T(x)-x\|^2 \, d\mu(x).
\end{equation*}
For simplicity, we denote the norm on $\LL^2(\R^n,\mu)$ by $\|f\|_\mu^2:=\int_{\R^n}\|f(x)\|^2d\mu(x)$. Note that if $T_{\mu}^{\nu}$ exists, then
\begin{equation*}
    W_2(\mu,\nu) = \|T_{\mu}^{\nu} - \operatorname{id}\|_{\mu}.
\end{equation*}
Furthermore, \cite{brenier1991polar} shows that when $\mu$ is absolutely continuous with respect to the Lebesgue measure, the map $T_{\mu}^{\nu}$ is uniquely defined as the gradient of a convex function $\phi$, i.e.\ $T_{\mu}^{\nu} = \nabla \phi$ (up to an additive constant).


\subsection{Linearized optimal transport}\label{sec:LOT-intro}
Linearized optimal transport (LOT) \cite{gigli-2011,merigot20,park18,wei13} defines an embedding of $\mc{P}(\R^n)$ into the linear space $\LL^2(\R^n,\sigma)$, with $\sigma$ being a fixed reference measure. Under the assumption that the optimal transport map exists, the embedding is defined by $\mu \mapsto T_{\sigma}^{\mu}$. This embedding can be used as a feature space, for example, to classify subsets of $\mc{P}(\R^n)$, to linearly approximate the Wasserstein distance, or for fast Wasserstein barycenter computations \cite{aldroubi2021partitioning, khurana2022supervised,merigot20,moosmuller2020linear,park18}.

In particular, the LOT embedding defines a linearized Wasserstein-2 distance:
\begin{equation}\label{eq:W2-LOT}
    \WLOT(\mu,\nu) := \|T_{\sigma}^{\mu}-T_{\sigma}^{\nu}\|_{\sigma}.
\end{equation}
In certain settings, this linearized distance approximates the Wasserstein-2 distance. The strongest results can be obtained when the so-called \emph{compatibility condition} is satisfied:
\begin{definition}[Compatibility condition \cite{aldroubi2021partitioning,moosmuller2020linear,park18}]
Let $\sigma,\mu \in W_2(\R^n)\cap \mc{P}_{\ac}(\R^n)$. We say that the LOT embedding is compatible with the $\mu$-pushforward of a function $g\in \LL^2(\R^n,\mu)$ if
\begin{equation*}
    T_{\sigma}^{g_{\sharp}\mu} = g \circ T_{\sigma}^{\mu}.
\end{equation*}
\end{definition}
The compatibility condition describes an interaction between the optimal transport map and the pushforward operator, namely it requires invertability of the exponential map \cite{gigli-2011}.

When the compatibility condition holds for two functions $g_1,g_2$, then LOT is an isometry, i.e.\ 
$\WLOT({g_1}_{\sharp}\mfixAdd,{g_2}_{\sharp}\mfixAdd)=W_2({g_1}_{\sharp}\mfixAdd,{g_2}_{\sharp}\mfixAdd)$ as shown in \Cref{LEM:ExactCompat} and \cite{moosmuller2020linear,park18}.  In particular, this is the case when $g$ is either a shift or scaling, or a certain type of shearing \cite{khurana2022supervised,moosmuller2020linear,park18}.

We can furthermore consider a generalization to ``almost compatible'' functions, also termed $\eps$-compatible:
\begin{definition}[$\eps$-compatibility]\label{epsCompatDef}
Let $\sigma,\mu \in W_2(\R^n)\cap \mc{P}_{\ac}(\R^n)$. We say that $\mathcal{H}$ is $\eps$- compatible with respect to $\sigma$ and $\mu$, if for every $h \in \mathcal{H}$, there exists a compatible transformation $g$ such that $\Vert g - h  \Vert_\mu < \eps$, where $g \circ T_\sigma^\mu = T_\sigma^{g_\sharp \mu}$.
\end{definition}

We remark that compatibility is stable.  Similar to compatibility implying isometry, there exist results that imply $\varepsilon$-compatible transformations imply ``almost"-isometry between $\WLOT$ and $W_2$.  Some of these results are accounted for in \cite[Proposition 4.1]{moosmuller2020linear}; however, we also extend these almost-compatibility results in \Cref{THM:epsCompat}.  These results make use of the H\"older regularity bounds for $\WLOT$ of \cite{gigli-2011,merigot20}. We note that the ``isometry under compatibility'' result mentioned above is a direct consequence of the preceding proposition, namely by setting $\varepsilon=0$.

In this paper, we consider measures $\mu_i, i=1,\ldots,N$ of the form $\mu_i={h_i}_{\sharp}\mu$, where $\mu$ is a fixed \emph{template measure}, and $h\in \mathcal{H}$ with $\mathcal{H}$ a space of functions in $\LL^2(\mathbb{R}^n,\mu)$. This is similar to assumptions in \cite{aldroubi2021partitioning,khurana2022supervised,moosmuller2020linear,park18}, where $\mathcal{H}$ consists of shifts and scalings, compatible maps, or has other properties, such as convexity and compactness. We will write $\mu_i \sim \mathcal{H}_\sharp \mu$ to indicated that $\mu_i$ is of such a form for all $i=1,\ldots,N$, and $\mathcal{H}$ will be specified in the respective context. Note that \cite{aldroubi2021partitioning} calls this data generation process an ``algebraic generative model''.


\subsection{Optimal transport with plug-in estimators}
Explicit descriptions of the measures $\mu$ are often unavailable in applications, and one must instead deal with finite samples of the measure. In this paper, we consider empirical distributions
\begin{equation}\label{eq:estimator}
    \widehat{\mu} = \frac{1}{k} \sum_{i=1}^k\delta_{Y_i}
\end{equation}
with $Y_i \sim \mu$ i.i.d. In what follows, we will consider approximations of both the target and reference distributions via empirical distributions.

The Kantorovich problem \eqref{eq:W2-metric} has a (possibly non-unique) solution for transporting an absolutely continuous measure $\sigma$ to an empirical measure of the form \eqref{eq:estimator}. Following \cite{deb2021rates}, we define the set of Kantorovich plans
\begin{equation}\label{EQ:KantProb}
    \Gamma_{\min} := \underset{\pi \in \Gamma(\sigma, \widehat{\mu})}\argmin \int_{\R^{2n}} \Vert x - y \Vert^2 d\pi(x,y),
\end{equation}
which may contain more than one transport plan.  In practice, these optimal transport plans are exactly computed via linear programming to solve \eqref{EQ:KantProb}.  We call optimal transport plans solved with linear programming $\gamma_{LP}$. It is much faster, however, to approximate the optimal transport plan by using an entropic regularized plan \cite{cuturi2013sinkhorn}.  In particular, we get a unique solution by solving
\begin{equation}\label{EQ:SinkProb}
    \gamma_{\beta} := \underset{\pi \in \Gamma(\sigma, \widehat{\mu})}\argmin \int \frac12 \Vert x - y \Vert^2 d\pi(x,y) + \beta D_{\textnormal{KL}}(\pi \Vert \sigma \otimes \widehat{\mu} ),
\end{equation}
where $D_{\textnormal{KL}}$
is the Kullback--Leibler divergence of measures \cite{joyce2011kullback}, $\sigma\otimes \widehat{\mu}$ is the direct product measure on the product space $\R^n\times\R^n$, and $\beta$ denotes the regularizer.  We solve \eqref{EQ:SinkProb} with Sinkhorn's algorithm, which yields entropic potentials $f_\beta$ and $g_\beta$ corresponding to $\sigma$ and $\widehat{\mu}$, respectively.

Regardless of whether we solve the optimal transport plan using \eqref{EQ:KantProb} or \eqref{EQ:SinkProb}, we can make a transport plan $\gamma \in \Gamma$ into a map by defining the barycentric projection
\begin{align}\label{baryEst}
    T_\sigma^{\widehat{\mu}} (x; \gamma) := \frac{ \int_y y d\gamma(x,y) }{ \int_y d\gamma(x,y) }, \hspace{0.3cm} \text{for } x \in \supp(\sigma).
\end{align}
\begin{remark}
    \corr{Notice that if $\mu = \widehat{\mu}$ (i.e. $\mu$ is a finite atomic measure), then $T_\sigma^\mu$ and $T_\sigma^{\widehat{\mu}}$ are the same $\sigma$-almost everywhere assuming $\sigma$ is the absolutely continuous.  If $\sigma$ is absolutely continuous, then the Kantorovich solution can be written as a product measure with the Monge solution as one of the products.  Thus, since no mass from $\sigma$ splits, the barycentric projection $T_\sigma^{\widehat{\mu}}$ is the same as $T_\sigma^\mu$ $\sigma$-almost everywhere.}
\end{remark}

This leads to a natural way to consider linearized Wasserstein-2 distances of the form \eqref{eq:W2-LOT} with absolutely continuous reference $\sigma$, and for empirical distributions:
\begin{equation}\label{eq:W2-LOT-empir}
    \WLOT(\widehat{\mu},\widehat{\nu}; \gamma) := \|T_\sigma^{\widehat{\mu}}(\ \cdot\ ; \gamma_{\widehat{\mu}} ) -T_\sigma^{\widehat{\nu}}(\  \cdot\ ; \gamma_{\widehat{\nu}} ) \|_{\sigma},
\end{equation}
where $\gamma \in \{\gamma_{LP}, \gamma_\beta \}$ denotes the method used to calculate the transport plans $\gamma_{\widehat{\mu}}$ and $\gamma_{\widehat{\nu}}$, which are transport plans from $\sigma$ to $\widehat{\mu}$ and $\widehat{\nu}$, respectively.  We suppress this notation and will simply use $T_\sigma^{\widehat{\mu}}(\ \cdot\ ; \gamma_{LP})$ or $T_\sigma^{\widehat{\mu}}(\ \cdot\ ; \gamma_\beta)$ to denote the barycentric projection map computed via linear programming and Sinkhorn, respectively, so that $\gamma_{LP}$ and $\gamma_{\beta}$ are understood to be in $\Gamma(\sigma, \widehat{\mu})$.  \corr{Notice that the solution $\gamma_{LP}$ is not necessarily unique.  In this case, the results we derive for $\WLOT$ and $T_\sigma^{\widehat{\mu}}(\ \cdot\ ; \gamma_{LP})$ still work with high probability as we use concentration inequalities and results that hold with high probability.}

To account for $m$ finite samples of the reference distribution, we define the empirical linearized Wasserstein-2 distance by
\begin{equation}\label{eq:W2-LOT-empir-all}
    \WLOTest(\widehat{\mu},\widehat{\nu}; \gamma) :=\left(\frac{1}{m}\sum_{j=1}^m \|T_\sigma^{\widehat{\mu}}(X_j ; \gamma_{\widehat{\mu}}) - T_\sigma^{\widehat{\nu}}(X_j ; \gamma_{\widehat{\nu}} )\|^2\right)^{1/2},
\end{equation}
where $X_j \sim \sigma$ i.i.d.

\begin{remark}
When we use $\gamma_{\beta}$ for a transport plan between $\widehat{\sigma}$ and $\widehat{\mu}$, note that our barycentric projection map is given by
\begin{align}\label{entrEst}
    T_{\widehat{\sigma}}^{\widehat{\mu}}(x ; \gamma_\beta) := \frac{ \frac{1}{k} \sum_{i=1}^k y_i \exp\bigg( \Big( g_{\beta, k}(y_i) - \frac12 \Vert x - y_i \Vert^2 \Big)/\beta \bigg)  }{ \frac{1}{k} \sum_{i=1}^k \exp\bigg( \Big( g_{\beta, k}(y_i) - \frac12 \Vert x - y_i \Vert^2 \Big)/\beta \bigg)  },
\end{align}
where $g_{\beta, k}$ denotes the entropic potential corresponding to $\widehat{\mu}$, $y_i \in \text{supp}(\widehat{\mu})$, and $k$ is the sample size for both $\widehat{\sigma}$ and $\widehat{\mu}$.
\end{remark}

\begin{remark}
Since our approximations will require us to use $m$ samples from the reference distributions, the barycentric projection map $T_\sigma^{\widehat{\mu}}(x)$ will only work for $x \in \supp(\widehat{\sigma})$; however, for general computation, we can just interpolate to calculate $T_\sigma^{\widehat{\mu}}(x)$ for $x \in \supp(\sigma) \setminus \supp(\widehat{\sigma})$.
\end{remark}

In what follows, we are interested in bounds for
\begin{equation*}
    |W_2(\mu,\nu)^2-\WLOTest(\widehat{\mu},\widehat{\nu}; \gamma)^2|
\end{equation*}
for $\gamma \in \{ \gamma_{LP}, \gamma_\beta \}$.  In particular, we want similar results to \Cref{THM:epsCompat} (Wasserstein-2 compared to LOT) and results in \cite{deb2021rates} (Wasserstein-2 compared to Wasserstein-2 on empirical distributions). This requires comparisons between all of $W_2(\mu,\nu)$, $\WLOT(\mu,\nu)$, $\WLOT(\widehat{\mu},\widehat{\nu}; \gamma),$ and $\WLOTest(\widehat{\mu},\widehat{\nu}; \gamma)$, which are discussed in \Cref{sec:compact_support} and \Cref{sec:non_compact_support}.

\subsection{Wassmap}\label{sec:Wassmap}

Various generalizations of MDS have been explored \cite{cox2008multidimensional} including stress minimization, which is useful in graph drawing \cite{khoury2012drawing,miller2022spherical}, Isomap \cite{tenenbaum2000global} which replaces pairwise distance by a graph estimation of manifold geodesics, and is useful for embedding data from $d$--dimensional nonlinear manifolds in $\R^D$.  Wang et al.~\cite{wang2010optimal} utilized MDS with $\Delta_{ij}=W_2(\mu_i,\mu_j)^2$ for data considered as probability measures in Wasserstein space with applications to cell imaging and cancer detection. Subsequently, Hamm et al.~\cite{hamm2022wassmap} proved that several types of submanifolds of $W_2$ can be isometrically embedded via MDS with Wasserstein distances (as in \cite{wang2010optimal}) and empirically studied Wassmap: a variant of Isomap that approximates nonlinear submanifolds of $W_2$. In particular, \cite{hamm2022wassmap} shows that for some submanifolds of $W_2(\R^m)$ of the form $\mathcal{H}_\sharp \mu$ where $\mathcal{H} = \{h_\theta:\theta\in\Theta\subset\R^d\}$ which are isometric Euclidean space, the parameter set $\Theta\subset\R^d$ can be recovered up to rigid transformation via MDS with Wasserstein distances (e.g., translations and anisotropic dilations).

\subsection{Other notations}

For scalars $a$ and $b$ we use $a\vee b$ to denote the maximum and $a\wedge b$ to denote the minimum value of the pair. Throughout the paper, constants will typically be denoted by $C$ and may change from line to line, and subscripts will be used to denote dependence on a given set of parameters. We use $a\asymp b$ to mean that $ca\leq b\leq Ca$ for some absolute constants $0<c,C<\infty$. 

For a random variable $X_n$, we say that $X_n = O_p(a_n)$ if for every $\eps > 0$ there exists $M > 0$ and $N > 0$ such that
\begin{align*}
    \mathbb{P}\left( \left| \frac{X_n}{a_n} \right| > M \right) < \eps \hspace{0.2cm} \forall n \geq N.
\end{align*}

We denote by $\mathcal{O}(d)$ the orthogonal group over $\R^d$, and the related Procrustes distance (in the Frobenius norm) between matrices $X,Y\in\R^{d\times N}$ is $\underset{Q\in\mathcal{O}(d)}\min\|X-QY\|_\fro$.

\section{LOT Wassmap algorithm and Main Theorem}\label{sec:maintheorem}

Here we present our main algorithm which is an LOT approximation to the Wassmap embedding of \cite{hamm2022wassmap}, and our main theorem which describes the quality of the embedding using some existing perturbation bounds for MDS.

\subsection{The LOT Wassmap Embedding Algorithm}

The algorithm presented here (\Cref{alg:lotWass}) takes discretized samples of a set of measures $\{\mu_i\}_{i=1}^N\subset W_2(\R^n)$ and a discretized sample of a reference measure $\sigma\in W_2(\R^n)$, computes transport maps from the empirical reference measure $\widehat{\sigma}$ to each empirical target measure $\widehat{\mu_i}$ using  optimal transport solvers and barycentric projections. Finally, the truncated right singular vectors and singular values of the centered transport map matrix are used to produce the low-dimensional embedding of the measures.  Two things are important to note here: first, the output of the algorithm is the same as the output of multi-dimensional scaling using pairwise squared LOT distances (or Sinkhorn distances in the approximate case), but we use the same trick as the reduction of PCA to the SVD to avoid actually computing the distance matrix; second, in contrast to the Wassmap embedding of~\cite{hamm2022wassmap} which requires $O(N^2)$ Wasserstein distance computations, \Cref{alg:lotWass} requires computation of only $O(N)$ optimal transport maps. Given the high cost of computing a single optimal transport map for densely sampled measures, this represents a significant savings.

Note that the factor of $\frac{1}{\sqrt{m}}$ appearing in the computation of the final embedding is due to \eqref{eq:W2-LOT-empir-all} where the $\frac{1}{m}$ appears in the definition of the empirical LOT distance. \Cref{LEM:gramDistLemm} shows that $T^\top T$ where $T$ is as in \Cref{alg:lotWass} is actually the MDS matrix $-\frac{1}{2}J\Lambda J$ where $\Lambda$ consists of the empirical LOT distances between the data, hence we absorb the $\frac{1}{m}$ into the norm in \eqref{eq:W2-LOT-empir-all} to get the matrix $T$ in \Cref{alg:lotWass}.  \corr{If we have a $\ell \times k$ matrix $A$, $\text{colstack}(A)$ is a $\ell \cdot k$ vector constructed by stacking the columns on top of each other.}








\begin{algorithm}
\caption{LOT WassMap Embedding}\label{alg:lotWass}

\SetKwInput{Input}{Input~}
\SetKwInOut{Output}{Output}
 \SetKw{Return}{Return}
\Input{~Reference point cloud $\{ w_i \}_{i=1}^m \sim \sigma \in \wspaceR$ \\
 Sample point clouds $\{ x_{j}^k \}_{j=1}^{n_k} \sim \mu_k \in \wspaceR$ ($k=1, \dotsc, N$) \\
OT solver (with regularizer if Sinkhorn) \\
Embedding dimension $d$}
\Output{Low-dimensional embedding points $\{z_i\}_{i=1}^N \subseteq \R^d$}

\For{$k = 1, \dots, N$}{
Calculate cost matrix $C_{ij} = \Vert w_i - x_{j}^k \Vert^2$\\
Compute OT plan $\gamma_k \in \R^{m \times n_k}$ between $\{w_i\}_{i=1}^m$ and $\{x_j^k\}_{j=1}^{n_k}$ using $C$ and OT solver\\
Calculate barycentric projection $\widetilde{T}_k(w_i) = \Big( \sum_{j=1}^{n_k} x_j^{k} (\gamma_k)_{ij} \Big) / \Big( \sum_{j=1}^{n_k} (\gamma_k)_{ij} \Big)$
}

\medskip
$\widehat{T} = \corr{\big[\text{colstack}\{\widetilde{T}_j(w_i) \}_{i=1}^m \big]_{j=1}^{N}}$ \\
\For{ $k=1, \dots, N$}{
$T_{: k} = \frac{1}{\sqrt{m}}(\widehat{T}_{:k} - \frac{1}{N} \sum_{k=1}^N \widehat{T}_{:k})$
}
Compute the truncated SVD of $T$ as $T_d = U_d\Sigma_dV_d^\top$\\
\Return{  $z_i = V_d\Sigma_d(i,:)$}
\end{algorithm}


\subsection{MDS Perturbation Bounds}
As stated above, the output of \Cref{alg:lotWass} is equivalent to the output of MDS on the transport map matrix $T$ therein. Consequently, the analysis of the algorithm will require some results regarding MDS. On the road to stating our main result, we summarize some nice MDS perturbation results of \cite{arias2020perturbation}. 



\begin{theorem}[{\cite[Theorem 1]{arias2020perturbation}}]\label{THM:MDSPerturbation}
Let $Y,Z\in\R^{d\times N}$ with $d<N$ such that $\rank(Y)=d$, and let $\eps^2:=\|Z^\top Z - Y^\top Y\|_{\textnormal{S}_p}$ for some $p\in[1,\infty]$. Then,
\[\min_{Q\in\mathcal{O}(d)}\|Z-QY\|_{\textnormal{S}_p} \leq \begin{cases}
\|Y^\dagger\|\eps^2+\left((1-\|Y^\dagger\|^2\eps^2)^{-\frac12}\|Y^\dagger\|\eps^2\right)\wedge d^\frac{1}{2p}\eps, & \|Y^\dagger\|\eps<1, \\
\|Y^\dagger\|\eps^2 + d^\frac{1}{2p}\eps, & \textnormal{o.w.}
\end{cases}\]
Consequently, if $\|Y^\dagger\|\eps \leq\frac{1}{\sqrt{2}}$, then
\[\min_{Q\in\mathcal{O}(d)}\|Z-QY\|_{\textnormal{S}_p} \leq (1+\sqrt{2})\|Y^\dagger\|\eps^2.\]
\end{theorem}

\begin{corollary}\label{COR:MDSPerturbation}
Let $y_1,\dots,y_N\in\R^d$ be centered, span $\R^d$, and have pairwise dissimilarities $\Delta_{ij}=\|y_i-y_j\|^2$.  Let $\{\Lambda_{ij}\}_{i,j=1}^N$ be arbitrary real numbers and $p\in[1,\infty]$.  If $\|Y^\dagger\|\|\Lambda-\Delta\|_{\textnormal{S}_p}^\frac{1}{2}\leq\frac{1}{\sqrt{2}}$, then MDS (\Cref{ALG:MDS}) with input dissimilarities $\{\Lambda_{ij}\}_{i,j=1}^N$ and embedding dimension d returns a point set $z_1,\dots,z_N\in\R^d$ satisfying
\[\min_{Q\in\mathcal{O}(d)}\|Z-QY\|_{\textnormal{S}_p} \leq (1+\sqrt{2})\|Y^\dagger\|\|\Lambda-\Delta\|_{\textnormal{S}_p}.\]
\end{corollary}

\begin{proof}[Proof of Corollary \ref{COR:MDSPerturbation}]
The proof follows along similar lines to that of \cite[Corollary 2]{arias2020perturbation} with some modifications. First, note that the centering matrix $J$ in MDS satisfies $\|J\|=1$ as it is an orthogonal projection. Then, by using the fact that $\|AB\|_{\textnormal{S}_p}\leq\|A\|\|B\|_{\textnormal{S}_p}$, we can estimate
\begin{equation}\label{EQN:aMDSPerturbation} \frac12\|J(\Lambda-\Delta)J\|_{\textnormal{S}_p} \leq \frac12\|J\|^2\|\Lambda-\Delta\|_{\textnormal{S}_p} \leq \frac12\|\Lambda-\Delta\|_{\textnormal{S}_p} < \sigma_d^2(Y),\end{equation}
where the final inequality follows by assumption.

Since $Y$ is a centered point set, we have $Y^\top Y = JY^\top Y J = -\frac12 J\Delta J$ (Lemma \ref{LEM:gramDistLemm}). Thus by Weyl's inequality, the fact that $\|\cdot\|\leq \|\cdot\|_{\textnormal{S}_p}$ for all $p$, and \eqref{EQN:aMDSPerturbation},
\begin{align*}\sigma_d\left(-\frac12 J\Lambda J\right) & \geq \sigma_d\left(-\frac12 J\Delta J\right) - \frac12\|J(\Lambda-\Delta)J\|_{\textnormal{S}_p} \\  & = \sigma_d^2(Y) - \frac12\|J(\Lambda-\Delta)J\|_{\textnormal{S}_p} \\& >0.\end{align*}
Consequently, $-\frac12 J\Lambda J$ has rank at least $d$, so if $Z$ contains the columns of the MDS embedding corresponding to $\Lambda$, then $Z^\top Z$ is the best rank-$d$ approximation of $-\frac12 J\Lambda J$ (by construction).  It follows from Mirsky's inequality and the facts that $-\frac12 J\Delta J=Y^\top Y$ and $\rank(Y)=d$ that
\begin{equation}\label{EQN:ZZMDSPerturbation} 
\left\|Z^\top Z + \frac12 J\Lambda J\right\|_{\textnormal{S}_p} \leq \left\|\frac12 J(\Lambda - \Delta)J\right\|_{\textnormal{S}_p}.
\end{equation}

Combining \eqref{EQN:aMDSPerturbation} and \eqref{EQN:ZZMDSPerturbation}, we have
\begin{align*}\eps^2:=\|Z^\top Z - Y^\top Y\|_{\textnormal{S}_p} \leq \left\|Z^\top Z + \frac12 J\Lambda J\right\|_{\textnormal{S}_p} + \left\|\frac12J(\Lambda-\Delta)J\right\|_{\textnormal{S}_p} & \leq \|J(\Lambda - \Delta)J\|_{\textnormal{S}_p}\\ & \leq \|\Lambda-\Delta\|_{\textnormal{S}_p}.\end{align*}
Thus, $\|Y^\dagger\|\eps\leq \|Y^\dagger\|\|\Lambda-\Delta\|_{\textnormal{S}_p}^\frac{1}{2}\leq\frac{1}{\sqrt{2}}$, so we may apply the final bound of Theorem \ref{THM:MDSPerturbation} to yield the conclusion.
\end{proof}



\subsection{Main Theorem} 

The following theorem shows the quality of an MDS embedding of a discrete subset of $W_2(\R^n)$ when approximations of the pairwise $W_2(\R^n)$ distances are used (via, for example, LOT approximations, Sinkhorn regularization, or other approximation techniques). The embedding quality is understood in two parts: first, how far away the set is from a subset of $W_2(\R^n)$ that is isometric to $\R^d$, and second, how good an approximation to the Wasserstein distances one utilizes in MDS.  The second source of error can always be made arbitrarily small given sufficient computation time or judicious choice of parameters (as in Sinkhorn, for example).  However, the first source of error arises from the geometry of the set of points, and may or may not be small.  

Note that using \Cref{COR:MDSPerturbation} outright would require computing a proxy distance matrix and applying MDS; however, to make \Cref{alg:lotWass} computationally efficient, we instead compute the truncated SVD of the centered transport maps rather than on the distance matrix between the transport maps. These are the same in theory, but allow for significantly less computation in practice. Below, we state our main theorem, which is stated in terms of the output of MDS on an estimation of Wasserstein distances between measures; but we stress that we are able to easily transfer the bounds to the output of \Cref{alg:lotWass}, which does not require any distance matrix computation.

\begin{theorem}\label{THM:Meta}
Let $\{\mu_i\}_{i=1}^N\subset W_2(\R^n)$. Suppose $\mathcal{W}\subset W_2(\R^n)$ is a subset of Wasserstein space that is isometric to a subset of Euclidean space $\Xi\subset\R^d$, and $\{\nu_i\}_{i=1}^N\subset\mc{W}$ and $\{y_i\}\subset\Xi$ are such that $|y_i-y_j| = W_2(\nu_i,\nu_j)$. Let $\Delta_{ij}:=W_2(\nu_i,\nu_j)^2$, $\Gamma_{ij}:=W_2(\mu_i,\mu_j)^2$, and $\Lambda_{ij}:=\lambda_{ij}^2$ for some $\lambda_{ij}\in\R$. Let $\{z_i\}_{i=1}^N$ be the output of MDS (\Cref{ALG:MDS}) with input $\Lambda$.  

If $|W_2(\mu_i,\mu_j)^2-W_2(\nu_i,\nu_j)^2|\leq \tau_1$ and $|W_2(\mu_i,\mu_j)^2-\lambda_{ij}^2|\leq \tau_2$ for some $\tau_1$ and $\tau_2$, and 
\begin{equation}\label{EQN:MetaThmBound}\|Y^\dagger\|\sqrt{N}\left(\tau_1+\tau_2\right)^\frac12\leq \frac{1}{\sqrt{2}},\end{equation}
then $\{z_i\}_{i=1}^N\subset\R^d$ satisfies
\[\min_{Q\in\mathcal{O}(d)}\|Z-QY\|_\fro \leq (1+\sqrt{2})\|Y^\dagger\|N\left(\tau_1+\tau_2\right).\]
\end{theorem}


\begin{proof}
Note that 
\[\|\Lambda-\Delta\|_\fro \leq \|\Gamma-\Delta\|_\fro+\|\Lambda-\Gamma\|_\fro\leq N(\tau_1+\tau_2).\]
Consequently, \eqref{EQN:MetaThmBound} allows us to apply \Cref{COR:MDSPerturbation} to yield the conclusion.
\end{proof}

Specializing \Cref{THM:Meta} to the case of \Cref{alg:lotWass} yields the following corollary, which shows that the truncated SVD of the centered LOT transport matrix $T$ is equivalent to the output $z_i$ of MDS in \Cref{THM:Meta}.

\begin{corollary}\label{COR:LOTWassGuarantee}
Invoke the notations and assumptions of \Cref{THM:Meta}. 
Choose a reference measure $\sigma \in W_2(\R^n)$ and compute all transport maps $T_{\sigma}^{\mu_i}$.
Let $T$ be the transport map matrix created by centering and column-stacking the transport maps $T_{\sigma}^{\mu_i}$ as in \Cref{alg:lotWass}. 
Let $ U_d\Sigma_d V_d^\top$ be the truncated SVD of $T$, and let $z_i = V_d\Sigma_d(i,:)$ for $1\leq i\leq N$ (i.e., $z_i$ is the output of \Cref{alg:lotWass}). Then $z_i$ is the output of MDS with $\Lambda$ being the empirical linearized Wasserstein-2 distance \eqref{eq:W2-LOT-empir-all}, and if \eqref{EQN:MetaThmBound} holds, then 
\[\min_{Q\in\mathcal{O}(d)}\|Z-QY\|_\fro \leq (1+\sqrt{2})\|Y^\dagger\|N\left(\tau_1+\tau_2\right).\]
\end{corollary}

\begin{proof}
Since $T$ is centered, \Cref{LEM:gramDistLemm} implies that $T^\top T=JT^\top TJ = -\frac12 J\Lambda J$. Consequently, if $-\frac12 J\Lambda J = V\Sigma^2 V^\top = T^\top T$, then $T$ has truncated SVD $T_d = U_d\Sigma_d V_d^\top$, and therefore $z_i = V_d\Sigma_d(i,:)$ arises from the truncated SVD of $T$ and is also the output of MDS with input $\Lambda$. The conclusion follows by direct application of \Cref{THM:Meta}.



\end{proof}

In the rest of the paper, we will discuss how various LOT approximations to Wasserstein distances affect the value of the bound $\tau_2$ appearing in \Cref{THM:Meta} and \Cref{COR:LOTWassGuarantee}.  In particular, we get different values of $\tau_2$ when we have compactly supported target measures (as in \Cref{THM:CptLipshitzLOT} for linear programming estimators and \Cref{THM:sinkEstThm} for Sinkhorn estimators) and non-compactly supported target measures (as in \Cref{THM:BaryNonCompact} for linear programming estimators and \Cref{THM:sinkNonCompact} for Sinkhorn estimators).

\section{Bounds for compactly supported target measures} \label{sec:compact_support}
To capture the bound $\tau_2$ of \Cref{THM:Meta}, we turn our attention to approximating the pairwise square-distance matrix $\begin{bmatrix} W_2^2(\mu_i, \mu_j) \end{bmatrix}_{i,j=1}^N$ appearing in the theorem statement with the finite sample, discretized LOT distance matrix that comes from differences between transport maps to a fixed reference, a finite sampling of $\mu_i$, and a discretization of the reference distribution $\sigma$.  In particular, the main approximation argument consists of the following triangle inequality:
\begin{align*}
     \left|W_2(\mu_1, \mu_2)^2 - \WLOTest(\widehat{\mu_1}, \widehat{\mu_2}; \gamma)^2 \right| &\leq \underbrace{\left| W_2(\mu_1, \mu_2)^2 - \WLOT(\mu_1,\mu_2)^2 \right|}_{\text{LOT error}} \\ & + 
    \underbrace{\left| \WLOT(\mu_1,\mu_2)^2 - \WLOT(\widehat{\mu_1},\widehat{\mu_2}; \gamma)^2 \right|}_{\text{finite sample and optimization error}} \\
    & + \underbrace{\left| \WLOT(\widehat{\mu_1},\widehat{\mu_2}; \gamma)^2 - \WLOTest(\widehat{\mu_1},\widehat{\mu_2} ; \gamma)^2 \right|}_{\text{discretized $\sigma$ sampling error}}.
\end{align*}
There are four sources of error between these two distance matrices:
\begin{enumerate}
    \item approximating the Wasserstein distance with LOT distance,
    \item approximating LOT embeddings between $\mu_i$ and $\mu_j$ with the barycenteric approximations computed using finite samples $\widehat{\mu_i}$ and $\widehat{\mu_j}$,
    \item approximating the integral with respect to the reference measure $\sigma$ by the discretized sampling $\widehat{\sigma}$, and
    \item optimization error in approximating the optimal transport map.
\end{enumerate}

\noindent The error from (1) and (3) are handled in \Cref{ap:estimatorapprox} whilst the error from (2) gives us the main theorems of this section.  Error from (4) is also implicitly considered by handling error from (2) since the optimization error for using a linear programming optimizer versus a Sinkhorn optimizer is seen in the error bounds of \Cref{THM:CptLipshitzLOT} and \Cref{THM:sinkEstThm}. 
 We deal with each error separately and chain the bounds together at the end.

Before dealing with any of the details of the proofs, we need the following assumptions on $\sigma$, $\mu$, and $\mathcal{H}$: 
\begin{assumption}\label{AS:epsCompatAssump} 
Consider the following conditions on $\sigma$, $\mu$, and $\mathcal{H}$
\begin{enumerate}[(i)]
    \item $\sigma \in \mathcal{P}_{ac}(\Omega)$ for a compact convex set $\Omega \subseteq B(0,R)\subset \R^n$ with probability density $f_\sigma$ bounded above and below by positive constants.
    
    \item $\mu$ has finite $p$-th moment with bound $M_p$ with $p > n$ and $p\geq 4$.

    \item There exist $a,A> 0$ such that every $h \in \mathcal{H}$ satisfies $a \Vert x \Vert \leq \Vert h (x) \Vert \leq A \Vert x \Vert$ \corr{for every $x \in \text{supp}(\mu)$}.

    \item $\mathcal{H}$ is compact \corr{with respect to the $L^2(\mu)$-norm} and $\eps$-compatible with respect to $\sigma,\mu \in W_2(\R^n)$.  Moreover, $\sup_{h,h' \in \mathcal{H}} \Vert h - h' \Vert_\mu \leq M$.

    \item $\mu_i \sim \mathcal{H}_{\sharp}\mu$ i.i.d.
\end{enumerate}
\end{assumption}

These assumptions ensure that $\eps$-compatible transformations are also ``$\eps$-isometric" as shown in \Cref{THM:epsCompat}. \corr{Moreover, note that $M \neq M_p$ as $M_p$ is important for extending theory to non-compactly supported measures while $M$ handles bounds associated with the complexity of the function class $\mathcal{H}$ of pushforwards.}

\subsection{Using the Linear Program to compute transport maps}

In this subsection, we assume that the classical linear program is used to compute the optimal transport maps from $\widehat{\mu_i}$ to the reference (and its discretization).

\begin{theorem}\label{THM:CptLipshitzLOT}
    Let $\delta > 0$.  Along with \Cref{AS:epsCompatAssump} and that $\mu \in \mathcal{P}_{ac}(\Omega)$ for the $\Omega$ in \Cref{AS:epsCompatAssump} \corr{with simply connected support}, assume that
    \begin{enumerate}[(i)]
        \item $T_\sigma^{\mu_i}$ is Lipschitz. 

        \item We estimate $\mu_i$ with an empirical measure $\widehat{\mu_i}$ using $k$ samples and discretize $\sigma$ with $m$ samples.  Let our estimator be given by \eqref{baryEst} with $\gamma$ solved using linear programming.

    \end{enumerate}
    Then with probability at least $1 - \delta$, 
    \begin{multline}\label{EQ:CptLipshitzLOT}
        \left|W_2(\mu_1, \mu_2)^2 - \WLOTest(\widehat{\mu_1}, \widehat{\mu_2}; \gamma_{LP})^2\right| \leq (M + \corr{2AR}) \left( C\eps^{\frac{p}{6p+16n} } + 2 O_p( r_n^{(k)} \log( 1 + k)^{t_{n}} ) \right.\\\left. + R \sqrt{ \frac{2 \log(2/\delta)}{m} } \right).
    \end{multline}
Here $C$ is the constant from \Cref{THM:epsCompat} depending on $n,p,\Omega,M_p$, the constants $a$ and $A$ come from \Cref{AS:epsCompatAssump} \corr{(iii)}, and
\begin{align*}
    r_n^{(k)} = \begin{cases}
                    2k^{-1/2} & n = 2,3 \\
                    2 k^{-1/2} \log(1 + k) & n = 4 \\
                    2 k^{-2/n} & n \geq 5
                \end{cases}, \hspace{0.3cm} \corr{t_{n} = \begin{cases}
                    \frac52 & n \leq 4 \\
                    2(1 + n^{-1}) & n > 4 
                \end{cases}},
\end{align*}
so that $r_n^{(k)}$ is on the order of $k^{-1/n}$ and $t_n$ is on the order of $2(1+n^{-1})$.  In this case, $\tau_2$ of \Cref{COR:LOTWassGuarantee} is bounded above by the right-hand side of \eqref{EQ:CptLipshitzLOT}.
\end{theorem}

\begin{proof}
Note that the transport plan that we are using for the following proof is $\gamma_{LP}$.  Henceforth, we will suppress $\gamma_{LP}$ from the terms $\WLOTest(\widehat{\mu_1}, \widehat{\mu_2}; \gamma_{LP})$ and $T_\sigma^{\widehat{\mu_j}}(\cdot ; \gamma_{LP})$ for simplicity.

Since $|x^2-y^2|=|x+y||x-y|$, we need to bound both 
\begin{enumerate}[(a)]
    \item \label{proof:bound-plus} $\left|W_2(\mu_1, \mu_2) + \WLOTest(\widehat{\mu_1}, \widehat{\mu_2})\right|$,
    \item \label{proof:bound-minus} $\left|W_2(\mu_1, \mu_2) - \WLOTest(\widehat{\mu_1}, \widehat{\mu_2})\right|$.
\end{enumerate}
\emph{We start with \eqref{proof:bound-plus}}: Since both $\mu_1$ and $\mu_2$ are pushforwards of a fixed template distribution $\mu$, we know that $\mu_i={h_i}_{\sharp}\mu$, where by \cite[Eq.\ 2.1]{ambrosio2013user} and our assumptions, it follows that
\begin{align*}
    W_2(\mu_1, \mu_2) = W_2( {h_1}_{\sharp} \mu, {h_2}_{\sharp} \mu) \leq \Vert h_1 - h_2 \Vert_\mu \leq M.
\end{align*}
\corr{Moreover, since $\mu$ is compactly supported for $\Omega \subseteq B(0,R)$ and $\mu_i = (h_i)_\sharp \mu$ with $h_i \in \mathcal{H}$ and $a \Vert x \Vert \leq \Vert h_i (x) \Vert \leq A \Vert x \Vert$, we know that $\mu_i$ is compactly supported with $\supp(\mu_i) \subseteq B(0,A R)$ for all $i$.} This implies that 
\begin{align*}
    \WLOTest(\widehat{\mu_1}, \widehat{\mu_2}) = \left(\frac{1}{m} \sum_{j=1}^m \underbrace{|T_\sigma^{\widehat{\mu_1}}(X_j) - T_\sigma^{\widehat{\mu_2}}(X_j) |^2}_{\leq \corr{(2AR)^2} } \right)^{1/2} \leq \corr{2AR}.
\end{align*}
Putting these estimates together, we have
\begin{align*}
    \left|W_2(\mu_1, \mu_2) + \WLOTest(\widehat{\mu_1}, \widehat{\mu_2})\right| \leq M + \corr{2AR}. 
\end{align*}
\emph{We continue with \eqref{proof:bound-minus}}:    
From the triangle inequality we get
\begin{align*}
    \left|W_2(\mu_1, \mu_2) - \WLOTest(\widehat{\mu_1}, \widehat{\mu_2}) \right| &\leq \left| W_2(\mu_1, \mu_2) - \WLOT(\mu_1,\mu_2) \right| + 
    \left| \WLOT(\mu_1,\mu_2) - \WLOT(\widehat{\mu_1},\widehat{\mu_2}) \right| \\
    & + \left| \WLOT(\widehat{\mu_1},\widehat{\mu_2}) - \WLOTest(\widehat{\mu_1},\widehat{\mu_2}) \right|.
\end{align*}

We now bound these three parts individually:
\begin{enumerate}[a)]
    \item By \Cref{AS:epsCompatAssump}, we can use $\eps$-compatibility of $\mathcal{H}$ in \Cref{THM:epsCompat} to get that
    \begin{align*}
        \left| W_2(\mu_1, \mu_2) - \WLOT(\mu_1,\mu_2) \right| \leq C \eps^{ \frac{p}{6p+16n}} ,
    \end{align*}
    where $C$ is from \Cref{THM:epsCompat}.
    \item For the second term, we again assume that any transport maps involving discrete measures are obtained from the linear program.  In particular, we see that
    \begin{align*}
        \WLOT(\mu_1,\mu_2) & = \|T_{\sigma}^{\mu_1}-T_{\sigma}^{\mu_2}\|_{\sigma}\\ &  \leq \|T_{\sigma}^{\mu_1} -T_{\sigma}^{\widehat{\mu}_1} \|_{\sigma} + \|T_{\sigma}^{\widehat{\mu}_1} - T_{\sigma}^{\widehat{\mu}_2}\|_{\sigma} + 
        \|T_{\sigma}^{\widehat{\mu}_2}-T_{\sigma}^{\mu_2}\|_{\sigma} \\
        & = \|T_{\sigma}^{\mu_1}-T_{\sigma}^{\widehat{\mu}_1}\|_{\sigma} + \|T_{\sigma}^{\widehat{\mu}_2}-T_{\sigma}^{\mu_2}\|_{\sigma} + 
        \WLOT(\widehat{\mu_1},\widehat{\mu_2}).
    \end{align*}
    \corr{Notice that we can equivalently apply the triangle inequalities starting from $\WLOT(\widehat{\mu_1},\widehat{\mu_2})$ to get
    \begin{align*}
        \WLOT(\widehat{\mu_1},\widehat{\mu_2}) &= \|T_{\sigma}^{\widehat{\mu}_1} - T_{\sigma}^{\widehat{\mu}_2}\|_{\sigma} \\
        &\leq \|T_{\sigma}^{\mu_1}-T_{\sigma}^{\mu_2}\|_{\sigma} + \|T_{\sigma}^{\mu_1} -T_{\sigma}^{\widehat{\mu}_1} \|_{\sigma}  + 
        \|T_{\sigma}^{\widehat{\mu}_2}-T_{\sigma}^{\mu_2}\|_{\sigma}  \\
        & = \|T_{\sigma}^{\mu_1}-T_{\sigma}^{\widehat{\mu}_1}\|_{\sigma} + \|T_{\sigma}^{\widehat{\mu}_2}-T_{\sigma}^{\mu_2}\|_{\sigma} + \WLOT(\mu_1,\mu_2) .
    \end{align*}
    }
    Note that \Cref{AS:epsCompatAssump}(i) implies that \corr{for every} $t > 0$ and $\alpha > 0$, $\E_\sigma[\exp( t \Vert x \Vert^\alpha) ] < \infty$ as $\sigma$ is compactly supported.  Because $T_\sigma^{\mu_i}$ is $L$-Lipschitz, this allows us to use \Cref{debrate} \corr{and optimize the exponents $t_{n,\alpha}$ over $\alpha>0$} to conclude that
    \begin{align*}
        |\WLOT(\mu_1,\mu_2) - \WLOT(\widehat{\mu_1},\widehat{\mu_2})|& \leq 
        \|T_{\sigma}^{\mu_1}-T_{\sigma}^{\widehat{\mu}_1}\|_{\sigma} + \|T_{\sigma}^{\widehat{\mu}_2}-T_{\sigma}^{\mu_2}\|_{\sigma} \\
        & \leq 2 \, O_p( r_n^{(k)}\log( 1 + k)^{t_n} ).
    \end{align*}
\item     
From \Cref{thm:mcdiarmid}
    we know that with probability at least $1 - \delta$,
    \begin{align*}
        \left|  \WLOT(\widehat{\mu_1}, \widehat{\mu_2}) - \WLOTest(\widehat{\mu_1}, \widehat{\mu_2}) \right| \leq R \sqrt{  \frac{ 2 \log( 2 / \delta )  }{ m }  }.
    \end{align*}
    
\end{enumerate}
    Putting these bounds together yields the result.
\end{proof}


\begin{remark}\label{rem:assumptions-H-for-Lipschitz}
    \corr{There exist conditions on $\mathcal{H}$ that ensure $T_\sigma^{\mu_i}$ is Lipschitz, which is the setting needed to apply \Cref{debrate}.  We cover these conditions in \Cref{SEC:CondCpt}. When $\mathcal{H}$ is not exactly compatible, we have \Cref{THM:baryCMPTCond_NonCompat} which requires more technical assumptions.  For the case when $\mathcal{H}$ is exactly compatible, we have \Cref{THM:baryCMPTCond} which only requires that $\mathcal{H}$ is comprised of Lipschitz functions.  Similar results in \Cref{SEC:CondCpt} hold for using Sinkhorn transport plans.}
\end{remark}

\subsection{Using entropic regularization (Sinkhorn) to compute transport maps}\label{sec:Sinkhorn}


Although \cite{deb2021rates} gives estimation rates in terms of a transport map constructed from solving the linear program associated to the optimal transport problem, solving the regularized optimal transport problem \eqref{EQ:SinkProb} and using the barycentric projection map \eqref{entrEst} is much faster.  For this section, we will assume that the target and reference measures are discretized with the same number of samples $k$.

\begin{remark}
Since we can choose $\sigma$ as well as the sample size for $\widehat{\sigma}$, we can allow $k = m$ in this case.  We believe, however, that choosing a larger sample size for $\sigma$ than $\mu_i$ (i.e. $m > k$) will result in better approximation.
\end{remark}

For the following results, we make use of the following quantity:

\begin{definition}
    Consider the Wasserstein geodesic between $\sigma=\mu_0$ and $\mu=\mu_1$ with $\mu_t$ being the measure on the geodesic for $t\in (0,1)$.  Let $f(t,x)$ be the density corresponding to $\mu_t$.  Then the integrated Fisher information along the Wasserstein geodesic between $\sigma$ and $\mu$ is given by
    \begin{align*}
        I_0(\sigma, \mu) = \int_0^1 \int_{\R^n} \Big\Vert \nabla_x \log f(t,x) \Big\Vert_2^2 f(t,x) dx dt.
    \end{align*}
\end{definition}

Moreover, recall that the convex conjugate of a function $\phi: \R^n\to\R$ is given by
\begin{align*}
    \phi^*(x^*) = \sup_{x \in \R^n} {x^*}^\top x - \phi(x),
\end{align*}
see, e.g., \cite[p. 45]{ambrosio2008gradient}.  Now by using Theorem 3 from \cite{pooladian2021}, we will show that under suitable conditions the entropic map $T_{\widehat{\sigma}}^{\widehat{\mu_i}}(\ \cdot \ ; \gamma_\beta)$ is close to $T_\sigma^{\mu_i}$.

\begin{theorem}[{\cite[Theorem 3]{pooladian2021}}]\label{pooladianThm}
    Assume that
    \begin{itemize}
        \item[(A1)]\label{A1} $\sigma, \mu_i \in \mathcal{P}_{ac}(\Omega)$ for a compact set $\Omega \subset \R^n$ with densities satisfying $f_\sigma, f_{\mu_i} \leq B$ and $f_{\mu_i} \geq b > 0$ for all $x \in \Omega$.
        
        \item[(A2)]\label{A2} $\phi \in C^2(\Omega)$ and $\phi^* \in C^{\alpha+1}(\Omega)$ for $\alpha > 1$, where $\phi^*$ denotes the convex conjugate of $\phi$.
        
        \item[(A3)]\label{A3} $T_\sigma^{\mu_i} = \nabla \phi$ with $m I \preceq \nabla^2 \phi(x) \preceq LI$ for $m, L > 0$ for all $x \in \Omega$.
    \end{itemize}
    Then the entropic map $T_{ \widehat{\sigma}}^{\widehat{\mu_i}}(\ \cdot \ ; \gamma_\beta)$ from $\widehat{\sigma}$ to $\widehat{\mu_i}$ with regularization parameter $\beta \asymp k^{- \frac{1}{n' + \widetilde{\alpha} + 1}}$ satisfies
    \begin{align*}
        \E \big\Vert T_{ \widehat{\sigma}}^{\widehat{\mu_i}}(\ \cdot \ ; \gamma_\beta) - T_{\sigma}^{\mu_i} \big\Vert_{\sigma}^2 \leq \big( 1 + I_0(\sigma, \mu_i) \big) k^{- \frac{\widetilde{\alpha} + 1}{2 n' + \widetilde{\alpha} + 1} } \log k, 
    \end{align*}
    where $n' = 2 \lceil n/2 \rceil$, $\widetilde{\alpha} = \alpha \land 3$, $k$ is the sample size for both $\widehat{\sigma}$ and $\widehat{\mu_i}$, and $I_0(\sigma, \mu_i)$ is the integrated Fisher information along the Wasserstein geodesic between $\sigma$ and $\mu_i$.
\end{theorem}

Given the sample size $k$ for both $\widehat{\sigma}$ and $\widehat{\mu_i}$, if we let
\begin{align*}
    Z_k = \Big\Vert T_{\widehat{\sigma}}^{\widehat{\mu_i}}(\ \cdot \ ; \gamma_\beta) - T_\sigma^{\mu_i} \Big\Vert_{\sigma},
\end{align*}
then by Jensen's inequality (for concave functions) and \Cref{pooladianThm} we have that
\begin{align*}
    \E[ Z_k ] \leq \E\big[ Z_k^2 \big]^{1/2} &\leq  \sqrt{ \big( 1 + I_0(\sigma, \mu_i) \big) k^{- \frac{\widetilde{\alpha} + 1}{2 n' + \widetilde{\alpha} + 1} } \log k } 
     \\
     &=  \sqrt{\log(k)(1 + I_0(\sigma, \mu_i))} k^{- \frac{\widetilde{\alpha} + 1}{2(2 n' + \widetilde{\alpha} + 1)} }.
\end{align*}
Now using Markov's inequality, we easily have the following corollary.
\begin{corollary}
    Assume that $\sigma$ and $\mu_i$ satisfy (A1)--(A3) of \Cref{A1} and let $\delta > 0$. Then with probability at least $1 - \delta$, we have that
    \begin{align*}
        \big\Vert T_{\widehat{\sigma} }^{\widehat{\mu_i}}(\ \cdot \ ; \gamma_\beta ) - T_\sigma^{\mu_i} \big\Vert_{\sigma} \leq \frac{1}{\delta} \sqrt{\log(k) \big( 1 + I_0(\sigma, \mu_i) \big) } k^{ -\frac{\widetilde{\alpha} + 1 }{ 2(2n' + \widetilde{\alpha}+ 1 ) }}  .
    \end{align*}
\end{corollary}

Now we can approximate $T_\sigma^{\mu_i}$ with the entropic map that is derived from using Sinkhorn's algorithm.  Although the barycentric projection map and entropic map approximations have similar rates of convergence, the entropic map is computationally faster at the cost of more stringent assumptions in the theorem. The main difference in assumptions below is the addition of (A1)--(A3) from \Cref{pooladianThm} and the asymptotic bound on the regularization parameter $\beta$ used in the entropic regularization.


\begin{theorem}\label{THM:sinkEstThm}
    Let $\delta > 0$.  Along with \Cref{AS:epsCompatAssump} and $\mu \in \mathcal{P}_{ac}(\Omega)$ for the $\Omega$ in \Cref{AS:epsCompatAssump}, assume that
    \begin{enumerate}[(i)]
        \item $\sigma$ and $\mu_i$ satisfy assumptions (A1)--(A3) from \Cref{pooladianThm} for all $i$.  Note that (A1), regularity of $\phi$ in (A2), and the upper bound of (A3) are satisfied under the conditions of Caffarelli's regularity theorem.
        
        \item Given empirical distributions $\widehat{\sigma}$ and $\widehat{\mu_i}$ both with $k$ sample size, assume that we have associated entropic potentials $(f_{\beta,k}, g_{\beta, k})$, where $\beta \asymp k^{- \frac{1}{n' + \widetilde{\alpha}+1}}$ and $n'$ and $\widetilde{\alpha}$ are defined in Theorem 3 from \cite{pooladian2021}.  Assume our estimator is $T_{\widehat{\sigma}}^{\widehat{\mu_i}}(\ \cdot \ ; \gamma_\beta)$ given by \eqref{entrEst}.
    \end{enumerate}
    Then with probability at least $1 - \delta$,
    \begin{align*}
        \left|W_2(\mu_i, \mu_j)^2 - \WLOTest(\widehat{\mu_i}, \widehat{\mu_j}; \gamma_\beta)^2 \right| \leq  & (M + 2\corr{A}R) \left( C \eps^{\frac{p}{6p+16n}} + \right. \\ & \left.  \frac{2}{\delta} \sqrt{ \log(k) ( 1 + I_0(\sigma, \mu_i))}  k^{- \frac{\widetilde{\alpha}+1}{ 2(2n' + \widetilde{\alpha} + 1)  } }   + R \sqrt{ \frac{2 \log(2/\delta)}{k} } \right).
    \end{align*}
    where $C$ is from \Cref{THM:epsCompat} and $I_0(\sigma,\mu_i)$ is defined in \Cref{pooladianThm}.  In this case, $\tau_2$ in \Cref{COR:LOTWassGuarantee} is bounded above by the right-hand side of the inequality above.
\end{theorem}
\begin{proof}
    Note that the transport plan that we are using for the following proof is $\gamma_{\beta}$.  Henceforth, we will suppress $\gamma_{\beta}$ from the notation $\WLOTest(\widehat{\mu_1}, \widehat{\mu_2}; \gamma_{\beta})$ for simplicity.
    
    Using the same reasoning as in \Cref{THM:CptLipshitzLOT}, we find that
    \begin{align*}
        \Big(W_2(\mu_i, \mu_j) + \WLOTest(\widehat{\mu_i}, \widehat{\mu_j}) \Big) \leq M + 2\corr{A}R. 
    \end{align*}
    Similar to the proof of \Cref{THM:CptLipshitzLOT}, we bound
    \begin{align*}
        \left|W_2(\mu_i, \mu_j) - \WLOTest(\widehat{\mu_i}, \widehat{\mu_j}) \right| &\leq \left| W_2(\mu_i, \mu_j) - \big\Vert T_\sigma^{{\mu_i}} - T_\sigma^{{\mu_j}} \big\Vert_\sigma \right| \\
        &+ \big\Vert T_\sigma^{\mu_i} - T_{\widehat{\sigma}}^{\widehat{\mu_i}}(\ \cdot \ ; \gamma_\beta) \big\Vert_\sigma + \big\Vert T_\sigma^{{\mu_j}} - T_{\widehat{\sigma}}^{\widehat{\mu_j}}(\ \cdot \ ; \gamma_\beta) \big\Vert_\sigma \\
        &+ \left|  \big\Vert T_{ \widehat{\sigma}}^{\widehat{\mu_i}}(\ \cdot \ ; \gamma_\beta) - T_{ \widehat{\sigma}}^{\widehat{\mu_j}}(\ \cdot \ ; \gamma_\beta) \big\Vert_\sigma - \WLOTest(\widehat{\mu_i}, \widehat{\mu_j}) \right|.
    \end{align*}
    The first and is bounded the same way as in the proof of \Cref{THM:CptLipshitzLOT} above.  \corr{For the last term, we apply \Cref{cor:mcdiarmidSink}.}  Since assumption (i) of \Cref{AS:epsCompatAssump}, implies assumption (A1) of \Cref{pooladianThm}, we get that with probability at least $1 - \delta$
    \begin{align*}
        \big\Vert T_\sigma^{{\mu_\ell}} - T_{\widehat{\sigma}}^{\widehat{\mu_\ell}}(\ \cdot \ ; \gamma_\beta) \big\Vert_\sigma \leq \frac{1}{\delta} \sqrt{\log(k) ( 1 + I_0(\sigma, \mu_\ell)) } k^{ -\frac{\widetilde{\alpha}+1}{ 2(2n' + \widetilde{\alpha} + 1)  } } 
    \end{align*}
    for $\ell = i$ and $\ell = j$.  Putting the bounds together, we get the result.
\end{proof}

Using \Cref{THM:CptLipshitzLOT} and \Cref{THM:sinkEstThm}, we see that as long as $\mu_i$ are $\eps$-compatible push-forwards of $\mu$ and the number of samples used in the empirical distribution is large enough, then our LOT distance is a computationally efficient and a tractable approximation for the Wasserstein distance and the distortion of the LOT Wassmap embedding of $\{\mu_i\}$ is small with high probability.  

\section{Bounds for non-compactly supported target measures}
\label{sec:noncompact-bounds}
\label{sec:non_compact_support}
In the last section, we saw that for compactly supported $\mu_i \sim \mathcal{H}_\sharp \mu$ (as well as a few other conditions), either the barycentric estimator $T_{\sigma}^{\widehat{\mu_i}}(\ \cdot \ ; \gamma_{LP})$ or the entropic estimator $T_{\sigma}^{\widehat{\mu_i}}(\ \cdot \ ; \gamma_\beta)$ will allow for fast yet accurate approximation of the pairwise Wasserstein distances $W_2(\mu_i, \mu_j)$, which in turn allows for fast, accurate LOT approximation to the Wassmap embedding~\cite{hamm2022wassmap} via \Cref{alg:lotWass}.  In this section, we show that we can adapt \Cref{THM:CptLipshitzLOT} and \Cref{THM:sinkEstThm} to non-compactly supported measures as long as we can approximate the non-compactly supported measure with a compactly supported and absolutely continuous measure.  To this end, we use the main theorem of \cite{merigot2021}.  

\begin{theorem}[\cite{merigot2021}]\label{THM:merigotTHM}
Let $\Omega$ be a compact convex set and let $\sigma$ be a probability density on $\Omega$, bounded from above and below by positive constants. Let $p > n$ and $p \geq 4$. Assume that $\mu, \nu \in \wspaceR$ have bounded $p$-th moment, and $\max(M_p( \mu ), M_p (\nu)) \leq M_p < \infty$. Then
\begin{align*}
    \Vert T_\sigma^\mu - T_\sigma^\nu \Vert_{\sigma} \leq C_{n, p, \Omega, M_p} W_1(\mu, \nu)^{\frac{p}{6p + 16n}}.
\end{align*}
\end{theorem}

To achieve our purposes, we will assume that $\mu$ is a non-compactly supported measure that has a suitable tail decay rate, and then show that there exists a compactly supported absolutely continuous $\widetilde{\mu}$ that approximates $\mu$ well (i.e., $W_1(\mu,\widetilde{\mu})<\eta.$).  \corr{The particular compactly supported $\widetilde{\mu}$ will be formed by a pushforward that is the identity on $B(0,R)$ and applies a modified version of the $\frac{x}{1 + \Vert x \Vert}$ map near the boundary of $B(0,R)$.} We achieve this in the following lemma.


\begin{lemma}\label{LEM:extLEMM1}
    Fix $\eta > 0$, and let $\sigma$ satisfy the assumptions of \Cref{THM:merigotTHM}. Moreover, let $\mu \in \wspaceR$ with density $f_\mu$ have a bounded $p$-th moment for some $p > n$ and $p \geq 4$.  Finally, assume that there exists some $R > 0$ such that for every $x \not\in B(0,R)$, we have
    \begin{align*}
        f_\mu(x) < \bigg( \frac{\eta}{ \corr{\widetilde{C}_{n, p, \Omega, M_p}} } \bigg)^{\frac{6p+16n}{p}} \frac{1}{ \Vert x \Vert^{n+2} }, 
    \end{align*}
    \corr{where $\widetilde{C}_{n, p, \Omega, M_p}$ denotes a combination of the constant from \Cref{THM:merigotTHM} and a constant from integrating over concentric $n$-spheres.}  Then there exists a compactly supported absolutely continuous measure $\widetilde{\mu}$ such that
    \begin{align*}
        \Vert T_\sigma^\mu - T_\sigma^{\widetilde{\mu}} \Vert_{\sigma} < \eta.
    \end{align*}
\end{lemma}

\corr{For many of the results before, we require that our compactly supported measures have a density bounded away from 0 and $\infty$.}  The next lemma will be useful in establishing conditions on $\mathcal{H}$ and $\mu$ so that our truncated measure $\widetilde{\mu}$ has a density that is bounded away from $0$.  \corr{In particular, we construct $\widetilde{\mu}$ by constructing a density that is compactly supported and remains bounded away from 0 and $\infty$ rather than using a pushforward.  The proof uses that a pushforward must exist between $\mu$ and $\widetilde{\mu}$.}

\begin{lemma}\label{LEM:extLEMM2}
    Let $\sigma$ satisfy the assumptions of \Cref{THM:merigotTHM} and let $\mu \in \wspaceR$ with density $f_\mu \leq C < \infty$ have a bounded $p$-th moment for some $p > n$ and $p \geq 4$.  Moreover, assume that there exists some $R > 0$ and $\eta > 0$ such that for $x \in B(0,R)$, we have $f_\mu(x) \geq c > 0$; and for every $x \notin B(0,R)$, we have
    \begin{align*}
        f_\mu(x) \leq \bigg( \frac{\eta}{\corr{\widetilde{C}_{n,p,\Omega, M_p}'} } \bigg)^{\frac{6p+16n}{p}} \frac{1}{ \Vert x \Vert^{n+2} } , 
    \end{align*}
    \corr{where $\widetilde{C}_{n,p,\Omega,M_p}'$ comes from combining the constant from \Cref{THM:merigotTHM}, a constant from integrating over concentric $n$-spheres, and another constant from our approximation method.}  Then there exists a compactly supported, absolutely continuous measure $\widetilde{\mu}$ with density $0 < c \leq b \leq f_{\widetilde{\mu}} \leq B < \infty$ such that
    \begin{align*}
        \Vert T_\sigma^\mu - T_\sigma^{\widetilde{\mu}} \Vert_{\sigma} < \eta.
    \end{align*}
\end{lemma}

The proofs of both \Cref{LEM:extLEMM1} and \Cref{LEM:extLEMM2} are located in \Cref{PF:extLEMM2}.  \corr{Notice that the condition on the density $f_\mu$ ensures that $W_1(\mu, \widetilde{\mu}) < \eta$ for the two different truncated measures $\widetilde{\mu}$.}  With these two lemmas above, we obtain the following theorems. Note that \Cref{THM:BaryNonCompact} replaces the assumption that $\mu$ is compactly supported with one of polynomial (in the ambient dimension) tail decay; while the second assumption below is the same as \Cref{THM:CptLipshitzLOT}, the final assumption differs from that of \Cref{THM:CptLipshitzLOT} by requiring the discretizations of $\sigma$ and $\mu_i$ to have the same sample size to apply the lemmas above.

\begin{theorem}\label{THM:BaryNonCompact}
    Let $\delta > 0$.  Along with \Cref{AS:epsCompatAssump}, assume that
    \begin{enumerate}[(i)]
        \item Every $\mu_i$ has bounded $p$-th moment for some $p > n$ and $p \geq 4$.  Moreover, assume that for all $i$, there exists some $R > 0$ such that for every $x \not\in B(0,\corr{A}R)$, we have
        \begin{align*}
            f_{\mu_i} < \bigg( \frac{\eta}{\corr{C} } \bigg)^{\frac{6p+16n}{p}} \frac{1}{  \Vert x \Vert^{n+2} },
        \end{align*}
        \corr{where $C = \widetilde{C}_{n, p, \Omega, M_p}$ or $C = \widetilde{C}_{n,p,\Omega, M_p}'$ depending on if we use the truncated measure $\widetilde{\mu}_i$ to be from \Cref{LEM:extLEMM1} or \Cref{LEM:extLEMM2} so that $W_1(\mu_i, \widetilde{\mu}_i) < \eta$}.

        \item $T_\sigma^{\widetilde{\mu}_i}$ is $L$-Lipschitz (this happens, e.g., if $\sigma$ and $\widetilde{\mu}_i$ are both compactly supported).

        \item Given empirical distributions $\widehat{\sigma}$ and $\widehat{\mu_i}$ with $\supp(\widehat{\mu_i}) \subseteq B(0,\corr{A}R)$ and sample sizes $m$ and $k$, respectively, let our estimator be the barycentric estimator \eqref{baryEst}, with $\gamma_{LP}$.
    \end{enumerate}
    Then with probability at least $1 - \delta$,
    \begin{multline*}
        \left|W_2(\mu_i, \mu_j)^2 - \WLOTest(\widehat{\mu_i}, \widehat{\mu_j}; \gamma_{LP})^2\right| \leq (M + 2\corr{A}R) \left( C\eps^{\frac{p}{6p+16n}} + 2 \eta +  2O_p( r_n^{(k)}\log( 1 + k)^{t_{n,\alpha}} ) \right.\\+ \left.R \sqrt{ \frac{2 \log(2/\delta)}{m} } \right),
    \end{multline*}
    where $r_n^{(k)}$ and $t_{n,\alpha}$ are defined in \Cref{THM:CptLipshitzLOT} and $C$ is a constant coming from \Cref{THM:epsCompat}.  In this case, $\tau_2$ of \Cref{COR:LOTWassGuarantee} is bounded above by the right-hand side of the inequality above.
\end{theorem}

Similarly for the entropic map case we have the following. Note that the primary difference in assumption between \Cref{THM:sinkNonCompact} and \Cref{THM:BaryNonCompact} is the addition of (A1)--(A3) from \Cref{pooladianThm} and the asymptotic assumption on the regularization parameter for the entropic map. The assumptions (i) and (ii) below are essentially the same as those of \Cref{THM:sinkEstThm}, but with $\widehat{\mu}_i$ replaced with $\widetilde{\mu}_i$ arising from \Cref{THM:BaryNonCompact}, whereas the additional assumptions below are that $\mu_i$ have decaying tails as opposed to being compactly supported.

\begin{theorem}\label{THM:sinkNonCompact}
    Let $\delta > 0$.  Along with \Cref{AS:epsCompatAssump} and (i) of \Cref{THM:BaryNonCompact}, assume that
    \begin{enumerate}[(i)]
        \item $\sigma$ and $\widetilde{\mu}_i$ satisfy assumptions (A1)--(A3) in \ref{A1} for all $i$, where $\widetilde{\mu}_i$ is the truncated measure from \Cref{THM:BaryNonCompact}.
        
        \item Given empirical distributions $\widehat{\sigma}$ and $\widehat{\mu_i}$ with $\supp(\widehat{\mu_i}) \subseteq B(0,\corr{A} R)$ and sample size $k$ for both, assume that we have associated entropic potentials $(f_{\beta,k}, g_{\beta, k})$, where $\beta \asymp k^{- \frac{1}{n' + \widetilde{\alpha}+1}}$ and $n'$ and $\widetilde{\alpha}$ are defined in \Cref{pooladianThm}.  Moreover, assume our estimator is given by \eqref{entrEst}.
    \end{enumerate}
    Then with probability at least $1 - \delta$,
    \begin{multline*}
        \left|W_2(\mu_i, \mu_j)^2 - \WLOTest(\widehat{\mu_i}, \widehat{\mu_j})^2\right| \leq (M + 2\corr{A}R) \left( C\eps^{\frac{p}{6p+16n}} + 2\eta + \right.  \\\left. \frac{2}{\delta} \sqrt{ \log(k) ( 1 + I_0(\sigma, \mu_i)) } k^{-\frac{\widetilde{\alpha}+1}{ 2(2n' + \widetilde{\alpha} + 1)  } }  + R \sqrt{ \frac{2 \log(2/\delta)}{k} } \right),
    \end{multline*}
    where $I_0(\sigma, \mu_i)$ is defined in \Cref{pooladianThm} and $C$ is a constant from \Cref{THM:epsCompat}.  In this case, $\tau_2$ of \Cref{COR:LOTWassGuarantee} is bounded above by the right-hand side of the inequality above.
\end{theorem}

The following is a proof for both theorems above.

\begin{proof}[Proof of \Cref{THM:BaryNonCompact,THM:sinkNonCompact}]
    In the following, we let $T_{\sigma}^{\widehat{\mu_i}}$ denote the optimal transport map estimator that we are considering (either the barycentric estimator with $\gamma_{LP}$ or the entropic estimator with $\gamma_\beta$) since the same proof works for both cases.  The only difference in the compactly supported case and these theorems is that our approximation now becomes
    \begin{align*}
        \left|W_2(\mu_i, \mu_j) - \WLOTest(\widehat{\mu_i}, \widehat{\mu_j}) \right| &\leq \left| W_2(\mu_i, \mu_j) - \|T_\sigma^{{\mu_i}} - T_\sigma^{{\mu_j}}\|_\sigma \right| \\& +  \left| \|T_\sigma^{{\mu_i}} - T_\sigma^{{\mu_j}}\|_\sigma - \|T_\sigma^{\widetilde{\mu_i}} - T_\sigma^{\widetilde{\mu_j}}\|_\sigma \right| \\
        &+ \left| \|T_\sigma^{\widetilde{\mu_i} } - T_\sigma^{\widetilde{\mu_j}}\|_\sigma - \|T_\sigma^{\widehat{\mu_i}} - T_\sigma^{\widehat{\mu_j}}\|_\sigma \right| \\&  + \left|  \|T_\sigma^{\widehat{\mu_i}} - T_\sigma^{\widehat{\mu_j}}\|_\sigma - \WLOTest(\widehat{\mu_i}, \widehat{\mu_j}) \right|,
    \end{align*}
    where $\widetilde{\mu}_i$ is defined as in the theorem statement and $\widehat{\mu_i}$ denotes the empirical measure of $\mu_i$.  \corr{For the following consider the truncation radius of $\widetilde{\mu}$ to be $AR$ rather than $R$ as was the case in \Cref{LEM:extLEMM1} and \Cref{LEM:extLEMM2}. We assume that $\supp(\widehat{\mu_i}) \subseteq B(0,\corr{A}R)$; thus, let us assume that we sample from $\mu_i$ conditioned that we restrict to $B(0,AR)$.  Since $\widetilde{\mu_i}\vert_{B(0,AR)} = \mu_i\vert_{B(0,AR)}$, we see that $\widehat{\mu_i}$ can equivalently be thought of as being sampled from $\widetilde{\mu}_i$ rather than $\mu_i$ conditioned that $\widehat{\mu_i} \subseteq B(0,AR)$.}  This means that the same bounds as before hold for most of the terms, while additionally,
    \begin{align*}
        \left| \|T_\sigma^{{\mu_i}} - T_\sigma^{{\mu_j}}\|_\sigma - \|T_\sigma^{\widetilde{\mu_i}} - T_\sigma^{\widetilde{\mu_j}}\|_\sigma \right| \leq \underbrace{\Vert T_{\sigma}^{\mu_i} - T_{\sigma}^{\widetilde{\mu}_i} \Vert_\sigma}_{\leq \eta} + \underbrace{\Vert T_{\sigma}^{\mu_j} - T_{\sigma}^{\widetilde{\mu}_i} \Vert_\sigma}_{\leq \eta} \leq 2 \eta.
    \end{align*}
    The rest of the terms are bounded the same exact way as before, and the result follows.
\end{proof}

In this section, we have shown that results for the case when the $\mu_i$ are compactly supported can be extended to non-compactly supported $\mu_i$ as long as their densities decay fast enough and the reference distribution $\sigma$ has a compact and convex support.

\section{Conditions on $\mathcal{H}$ and $\mu$ (Compact case)}\label{SEC:CondCpt}

In this section, we derive conditions on $\mathcal{H}$ and $\mu$ so that the assumptions of the theorems above are satisfied for $\mu_i \sim \mathcal{H}_\sharp \mu$.  In particular, we can break down our requirements on $\mathcal{H}$ and $\mu$ by noting the necessary conditions on $\mu_i$ for the barycentric map estimator and entropic map estimator separately.  For simplicity, we will assume that $\mathcal{H}$ is exactly compatible with respect to $\sigma$ and $\mu$.

\corr{In \Cref{THM:baryCMPTCond_NonCompat}, we first describe the conditions to have \Cref{THM:CptLipshitzLOT} hold when $\mathcal{H}$ is \textit{not} made of exactly-compatible transformations.  Next, we consider the easier case of exactly-compatible transformations in \Cref{THM:baryCMPTCond} and \Cref{THM:sinkCMPTCond} to ensure that \Cref{THM:CptLipshitzLOT} and \Cref{THM:sinkEstThm} works.}

\begin{theorem}[Barycentric Map Case (Compact, non-compatible)]\label{THM:baryCMPTCond_NonCompat}
    \corr{Let \Cref{AS:epsCompatAssump} hold true. If $\mu$ has simply connected compact support with density such that $0 < c \leq f_\mu \leq C < \infty$, $h_i \in \mathcal{H}$ is continuously differentiable with $L_1 I \preceq J_{h_i}(x)$ for some $L_1>0$, and $(h_i)_\sharp \mu$ has convex support, then $T_\sigma^{(h_i)_\sharp \mu}$ is Lipschitz.  Hence condition (i) of \Cref{THM:CptLipshitzLOT} holds.}
\end{theorem}

\begin{theorem}[Barycentric Map Case (Compact, compatible)]\label{THM:baryCMPTCond}
    Along with \Cref{AS:epsCompatAssump} (with $\varepsilon=0$ so that every $h\in\mathcal{H}$ is exactly compatible with $\sigma$ and $\mu$), assume that \corr{$h \in \mathcal{H}$ is continuous}, $\mu_i \sim \mathcal{H}_\sharp \mu$ i.i.d., and that 
     \begin{enumerate}[(i)]
        \item $\mu$ has \corr{and has simply connected compact supported} 
        \item $\sigma$ is chosen such that $T_\sigma^\mu$ is Lipschitz,
    \end{enumerate}
    Then $\mu_i$ satisfies the condition (i) of \Cref{THM:CptLipshitzLOT} i.e., $T_\sigma^{\mu_i}$ is Lipschitz.
\end{theorem}

For the entropic case, the assumptions on $\mu$ and $\sigma$ are the same, but we require an additional assumption regarding the Jacobian of elements of $\mathcal{H}$.

\begin{theorem}[Entropic Map Case (Compact)]\label{THM:sinkCMPTCond}
    Along with \Cref{AS:epsCompatAssump}, assume that
    \begin{enumerate}
        \item[(iii)] \corr{\textbf{Conditions for $\sigma$ and $\mu$:}} $\sigma$ and $\mu$ satisfy (A1)-(A3), \corr{and $\mu$ has simply connected compact support;} 

        \item[(iv)] \corr{\textbf{Conditions for $\mathcal{H}$:} for each $h \in \mathcal{H}$, we have $h$ is continuous and $L_1 I \preceq J_{h}(x)I$ for some $L_1 > 0$; moreover, each $h \in \mathcal{H}$ is exactly compatible (i.e., $\eps = 0$).}
        
    \end{enumerate}
     \corr{Then} $\mu_i$ satisfies the condition (i) of \Cref{THM:sinkEstThm}.
\end{theorem}

The proofs of \Cref{THM:baryCMPTCond_NonCompat}, \Cref{THM:baryCMPTCond}, and \Cref{THM:sinkCMPTCond} are given in \Cref{PF:baryCMPTCond}.

\section{Conditions on $\mathcal{H}$ and $\mu$ (Non-compact case)}\label{SEC:CondNonCpt}

For the non-compactly supported cases, we need to add assumptions that $\mathcal{H}$ is closed under inversion as well as lower and upper boundedness of the density $f_\mu$.  This gives us the following theorems.

\begin{theorem}[Barycentric Map Case (Non-Compact)]\label{THM:baryCondNonCompact}
Along with \Cref{AS:epsCompatAssump} (with $\varepsilon=0$ so that every $h\in\mathcal{H}$ is exactly compatible with $\sigma$ and $\mu$), assume that $\mu_i \sim \mathcal{H}_\sharp \mu$ i.i.d.  Assume further that
\begin{enumerate}[(i)]
    \item for every $h \in \mathcal{H}$, there exists an inverse $h^{-1}\in\mathcal{H}$.

    \item The density of $\mu$ is supported on all of $\R^n$ with $f_\mu (x) \leq C < \infty$ for all $x$, and $f_\mu(x) \geq c > 0$ for all $x \in B(0,\corr{AR})$.  Moreover, $f_\mu$ has a decay rate as in \Cref{LEM:extLEMM2} for $x \not\in B(0,R)$.
    
    \item \corr{every $h \in \mathcal{H}$, is bi-Lipschitz with $L_1 \vert x - y \vert \leq \vert h(x) - h(y) \vert \leq L_2 \vert x - y \vert$, where $a, A > 0$ is from \Cref{AS:epsCompatAssump}.}
    
\end{enumerate}
Then $\mu_i$ satisfies the conditions of \Cref{THM:BaryNonCompact}.

\end{theorem}

\begin{theorem}[Entropic Map Case (Non-Compact)]\label{THM:sinkCondNonCompact}
Assume that $\mu_i \sim \mathcal{H}_\sharp \mu$ i.i.d. and that $\mu$, $\mathcal{H}$, and $\sigma$ satisfy the conditions of \Cref{THM:baryCondNonCompact}. 
 Then $\mu_i$ satisfies the conditions of \Cref{THM:sinkNonCompact}.    
\end{theorem}

The proofs of both \Cref{THM:baryCondNonCompact,THM:sinkCondNonCompact} are found in \Cref{PF:baryCondNonCompact}.


\section{Experiments}\label{sec:experiments}
We demonstrate that \Cref{alg:lotWass} does in fact attain correct embeddings given finite sampling and without explicitly computing the pairwise Wasserstein distances.  We test both variants of our algorithm above using the linear program or entropic regularization to compute the transport maps from the data to the reference measure, and illustrate the quality of embeddings as well as the relative embedding error 
\begin{align*}
    \min_{Q} \frac{\|Y - QX\|_\fro}{\|Y\|_\fro}
\end{align*}
as a function of the sample size $m$ of the data and reference measures.

In all experiments, we generate $N$ data measures, $\mu_i$, which are Gaussians of various means and covariance, and a fixed reference measure $\sigma$ drawn from the standard normal distribution $\mathcal{N}(0,I)$. We randomly sample $m$ points from each measure to form the empirical measure, and random noise from a Wishart distribution is added to the covariance matrices of the data measures $\mu_i$. Additionally, in each experiment we compute the optimal rotation of the embeddings to properly align them with the true embedding and thus give an accurate error estimate for each trial.

For each experiment, we provide a figure for qualitative assessment of the embedding as well as a quantitative figure in which we compute the relative error as above for the embeddings as a function of $m$, the sample size used to generate the empirical data and reference measures. For the latter figures, we run 10 trials of the embedding and average the relative error; error bands showing one standard deviation are shown on each figure.  A jupyter notebook containing all of the experiments that generate the figures below can be found at \url{https://github.com/varunkhuran/LOTWassMap}. 

\subsection{Experiment 1: circle translation manifold}

First, we consider a 1-dimensional manifold of translations as follows. We uniformly choose $N=10$ points on the circle of radius 8, which we denote $x_i$, and each data measure $\mu_i$ is a Gaussian with mean $x_i$ and covariance matrix $\begin{bmatrix} 1 & -.5 \\ -.5 & 1\end{bmatrix}.$ Thus, our data set is a set of Gaussians translated around the circle. The Wishart noise added to the covariance matrix prior to sampling the $\mu_i$ is of the form $GG^\top$ where $G$ has i.i.d.~$\mathcal{N}(0,0.5)$ entries.  We choose the standard normal distribution $\mathcal{N}(0,I)$ as our reference measure $\sigma$. We randomly sample $m=1000$ points from each data measure and the reference measure independently.  \Cref{fig:circle} shows the original sampled data and the reference measure (in blue), the true embedding points $x_i$, and the embeddings of \Cref{alg:lotWass} when using the linear program and Sinkhorn with regularization parameter $\lambda=1$.

One can easily see that the embeddings are qualitatively good as expected given the theory above and the results of \cite{hamm2022wassmap} in similar experiments. \Cref{fig:circleerror} shows the relative error vs.~sampling size $m$ of the measures, and one can see the good performance for modest sample sizes.



\begin{figure}[h]
\centering
\includegraphics[width=\textwidth]{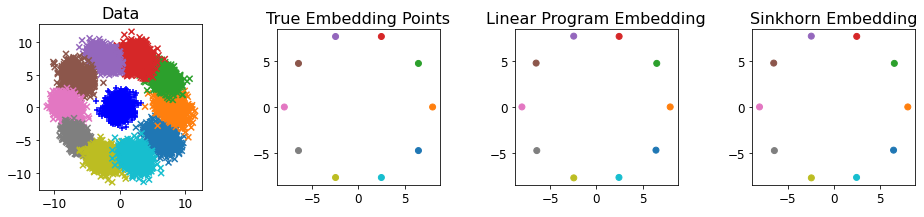}
    \caption{1-D Manifold of translations: \textbf{(Left)} reference measure $\sigma\sim\mathcal{N}(0,I)$ in blue and data measures $\mu_i$ which are Gaussians with the same covariance matrix and means $x_i$ uniformly sampled from the circle of radius $8$. \textbf{(Left Middle)} Means $x_i$ of $\mu_i$ which are the true embedding points. \textbf{(Right Middle)} Embedding attained with \Cref{alg:lotWass} using the linear program. \textbf{(Right)} Embedding attained with \Cref{alg:lotWass} using the Sinkhorn distance with $\lambda=1$.}
    \label{fig:circle}
\end{figure}

\begin{figure}[h]
\centering
\includegraphics[width=.8\textwidth]{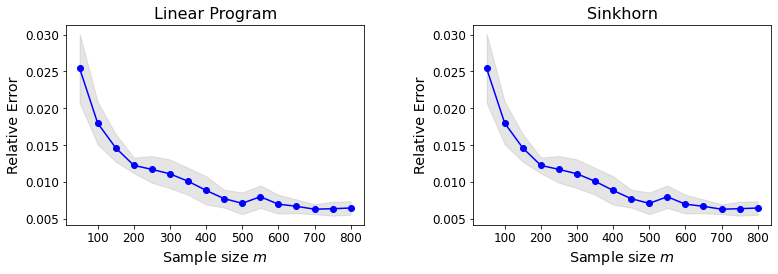}
\caption{Embedding error vs.~$m$ (number of sample points from data and reference distributions for the 1-D translation manifold. Optimal transport maps are computed via the Linear Program \textbf{(Left)} and Sinkhorn with $\lambda=1$ \textbf{(Right)}.}\label{fig:circleerror}
\end{figure}



\subsection{Experiment 2: rotation manifold}

Next, we consider a 1-dimensional rotation manifold in which we generate $N=10$ data measures of Gaussians whose means lie at uniform samples of the circle of radius $8$, which we denote $(8\cos\theta_i,8\sin\theta_i)$, and whose covariance matrices are rotations of $\begin{bmatrix}2 & 0 \\ 0 & .5\end{bmatrix}$ by the angles $\theta_i$. As in experiment 1, the noise level added is $0.5$ and we sample $m=1000$ points from each measure.  \Cref{fig:rotation} shows the data measures, true embedding, and embeddings from \Cref{alg:lotWass} using both the linear program and Sinkhorn (with $\lambda=1$) to compute the optimal transport maps. \Cref{fig:rotationerror} shows the relative error vs.~sample size.

\begin{figure}[h]
\centering
\includegraphics[width=\textwidth]{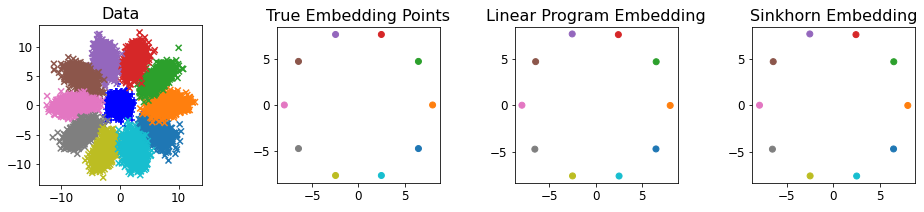}
    \caption{1-D Manifold of rotations: \textbf{(Left)} reference measure $\sigma\sim\mathcal{N}(0,I)$ in blue and data measures $\mu_i$ which are Gaussians with means lying on the circle of radius $8$ and covariance matrices that are rotations of each other. \textbf{(Left Middle)} Means $x_i$ of $\mu_i$ which are the true embedding points. \textbf{(Right Middle)} Embedding attained with \Cref{alg:lotWass} using the linear program. \textbf{(Right)} Embedding attained with \Cref{alg:lotWass} using the Sinkhorn distance with $\lambda=1$. }
    \label{fig:rotation}
\end{figure}

\begin{figure}[h]
\centering
\includegraphics[width=.8\textwidth]{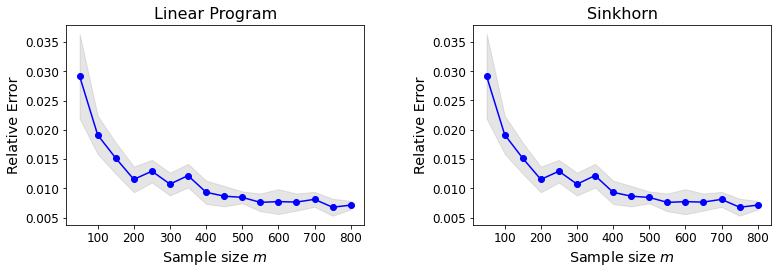}
\caption{Embedding error vs.~$m$ (number of sample points from data and reference distributions for the 1-D rotation manifold. Optimal transport maps are computed via the Linear Program \textbf{(Left)} and Sinkhorn with $\lambda=1$ \textbf{(Right)}.}\label{fig:rotationerror}
\end{figure}

\subsection{Experiment 3: grid translation manifold}

Here, we consider a 2-dimensional translation manifold in which we generate $N=25$ data measures of Gaussians whose means lie on a $5\times 5$ uniform grid on the cube $[-10,10]^2$ and which have constant covariance matrix $\begin{bmatrix} 1 & -.5 \\ -.5 & 1\end{bmatrix}.$ We sample $m=1000$ points from each measure and the noise level is again $0.5$. In the Sinkhorn embedding, we use regularization $\lambda=10$. \Cref{fig:translation,fig:translationerror} show the data, embeddings, and relative error vs.~sample size.

\begin{figure}[h]
\centering
\includegraphics[width=\textwidth]{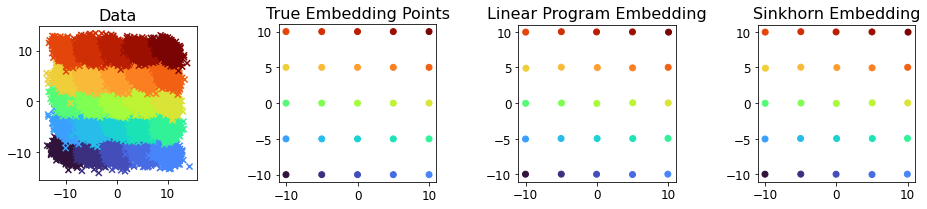}
    \caption{2-D Manifold of translations: \textbf{(Left)} data measures $\mu_i$ which are Gaussians with the same covariance matrix and means $x_i$ taken from a $5\times 5$ uniform grid on $[-10,10]^2$. \textbf{(Left Middle)} Means $x_i$ of $\mu_i$ which are the true embedding points. \textbf{(Right Middle)} Embedding attained with \Cref{alg:lotWass} using the linear program. \textbf{(Right)} Embedding attained with \Cref{alg:lotWass} using the Sinkhorn distance with $\lambda=10$.}
    \label{fig:translation}
\end{figure}

\begin{figure}[h]
\centering
\includegraphics[width=.8\textwidth]{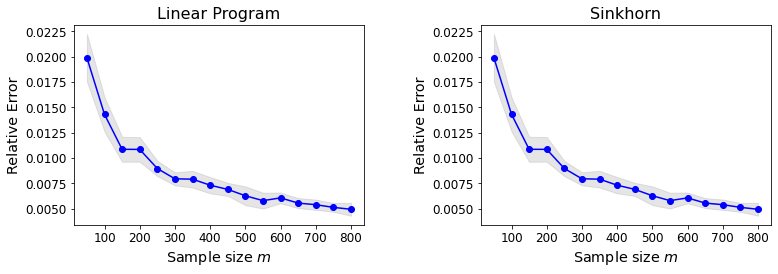}
\caption{Embedding error vs.~$m$ (number of sample points from data and reference distributions for the 2-D translation manifold. Optimal transport maps are computed via the Linear Program \textbf{(Left)} and Sinkhorn with $\lambda=10$ \textbf{(Right)}.}\label{fig:translationerror}
\end{figure}

\subsection{Experiment 4: Dilation manifold}

Here, we consider a 2-dimensional anisotropic dilation manifold in which we generate $N=9$ data measures of Gaussians with mean 0 and anisotropically scaled covariance matrices of the form $\textnormal{diag}(\alpha_i^2,\beta_i^2)$ for $(\alpha_i,\beta_i)$ taken from a uniform $3\times3$ grid on $[1,4]^2$. We sample $m=1000$ points from the reference measure and $n=2500$ points from the data measures and the noise level added to the covariance matrices is $0.5$ as before. In the Sinkhorn embedding, we use regularization $\lambda=100$. \Cref{fig:dilation} show the data measures, true embedding parameters, and embeddings from \Cref{alg:lotWass}. Note that the true embedding parameters are centered to allow them to be comparable to the output of \Cref{alg:lotWass} which are naturally centered.  

\Cref{fig:dilationerror} shows the relative error vs.~$m$, and for this experiment we choose $n=m$ so that the sampling order of the data and reference measure are the same. For this case, we see that the relative error of the embedding decays much more slowly than the previous experiments. One possible reason for this is that there is significant overlap in the distributions for the dilated measures, and to overcome this issue one may have to sample many more points in forming the empirical distribution so that the tails of the data measures are sampled more frequently. 

\begin{figure}[h]
\centering
\includegraphics[width=\textwidth]{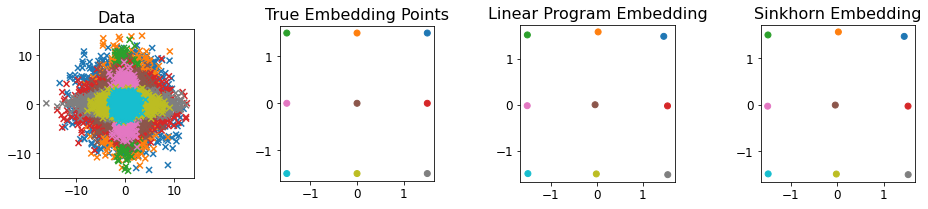}
    \caption{2-D Manifold of Anisotropic Dilations: \textbf{(Left)} data measures $\mu_i$ which are Gaussians with mean $0$ and anisotropically dilated covariance matrices where dilations are taken from a $3\times 3$ uniform grid on $[1,4]^2$. \textbf{(Left Middle)} Dilation factors $(x_i,y_i)$ of $\mu_i$ which are the true embedding points. \textbf{(Right Middle)} Embedding attained with \Cref{alg:lotWass} using the linear program. \textbf{(Right)} Embedding attained with \Cref{alg:lotWass} using the Sinkhorn distance with $\lambda=100$.}
    \label{fig:dilation}
\end{figure}

\begin{figure}[h]
\centering
\includegraphics[width=.8\textwidth]{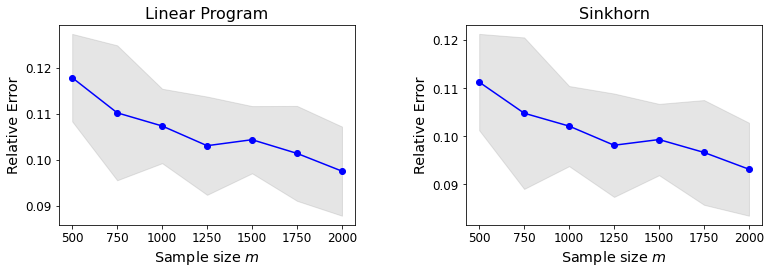}
\caption{Embedding error vs.~$m$ (number of sample points from data and reference distributions for the 2-D translation manifold. Optimal transport maps are computed via the Linear Program \textbf{(Left)} and Sinkhorn with $\lambda=10$ \textbf{(Right)}.}\label{fig:dilationerror}
\end{figure}

\subsection{Experiment 5: Time Comparison}

Here, we repeat Experiment 3 in which data measures are centered on a uniform grid and are translations of a fixed Gaussian measure. We plot the time it takes to compute the embedding via \Cref{alg:lotWass} using the Linear Program or Sinkhorn with $\lambda = 1$ and the Wassmap algorithm of \cite{hamm2022wassmap} which requires computing the entire square Wasserstein distance matrix $[W_2(\mu_i,\mu_j)]_{i,j=1}^N$ and the SVD of its centered version as in \Cref{alg:MDS}. For this experiment, we always choose $n=m$ so that the reference measure and data measure sampling rates are the same. One can easily see that a substantial gain in timing is achieved by LOT Wassmap, while previous experiments show that the quality of the embedding does not degrade significantly when LOT is used.

Finally, we plot the timing for the same experiment for the Linear Program and Sinkhorn with $\lambda=1$ and $\lambda=10$ for larger sample sizes to illustrate the character of these choices (\Cref{fig:timingsinkhorn}). As expected, larger regularization parameter yields faster computation time, though the difference is relatively small even for modestly large sample size.

\begin{figure}[h]
\centering
\includegraphics[width=.8\textwidth]{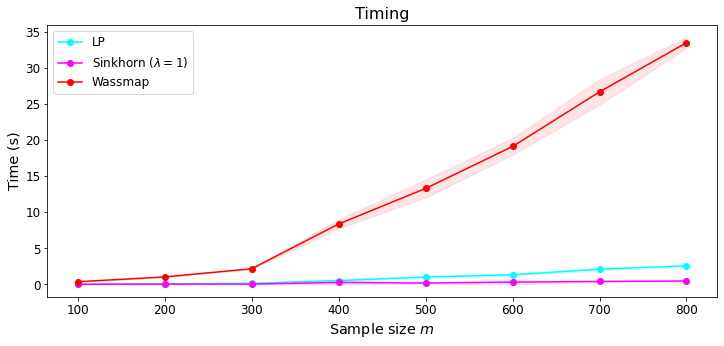}
\caption{Timing vs.~sample size $m$ of the reference distribution and data measures. The data set consists of $N=25$ measures translated on a $5\times5$ uniform grid on $[-10,10]^2$ as in Experiment 3. Shown are the computation times to compute the Wassmap embedding and the embeddings of \Cref{alg:lotWass} using the Linear Program (LP) and Sinkhorn with regularization parameter $\lambda=1$.}\label{fig:timing}
\end{figure}

\begin{figure}[h]
\centering
\includegraphics[width=.8\textwidth]{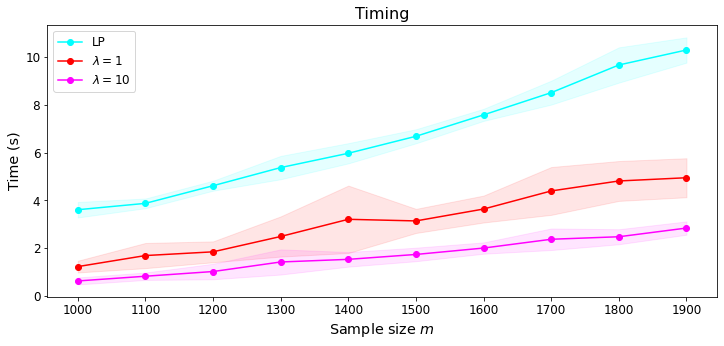}
\caption{Timing vs.~sample size $m$ of the reference distribution and data measures. The data set consists of $N=25$ measures translated on a $5\times5$ uniform grid on $[-10,10]^2$ as in Experiment 3. Shown are the computation times to compute the embeddings of \Cref{alg:lotWass} using the Linear Program (LP) and Sinkhorn with regularization parameters $\lambda=1$ and $\lambda=10$.}\label{fig:timingsinkhorn}
\end{figure}

\section*{Acknowledgements}

K.H.~acknowledges support from the UTA Research Enhancement Program from the College of Science at the University of Texas at Arlington, the Fields Institute for Research in Mathematical Sciences and the Army Research Office under Grant Number W911NF-23-1-0213. The views and conclusions contained in this document are those of the authors and should not be interpreted as representing the official policies, either expressed or implied, of the Army Research Office or the U.S. Government. The U.S. Government is authorized to reproduce and distribute reprints for
Government purposes notwithstanding any copyright notation herein. C.M.~is supported by NSF awards DMS-2306064, DMS-2410140 and by a seed grant from the School of Data Science and Society at UNC. A.C.~is partially supported by NSF awards DMS-2012266, CISE-2403452 and a gift from Intel research. K.H.~and A.C.~thank the Fields Institute and participants of the Focus Program on Data Science, Approximation Theory, and Harmonic Analysis for their hospitality, which facilitated the initial discussions of this research.

The authors thank the reviewer for their detailed feedback that improved the manuscript.

\bibliographystyle{plain}
\bibliography{wass}


\appendix
 
\section{Helper Theorems and Lemmas}\label{ap:helpertheo}

We use the following lemma to extend \Cref{COR:MDSPerturbation} to get our main theorem (\Cref{THM:Meta}). The proof follows standard arguments, e.g., as in \cite{mardia1979multivariate}; the proof is included for completeness.

\begin{lemma}[{\cite[Theorem 14.2.1]{mardia1979multivariate}}, for example]\label{LEM:gramDistLemm}
Consider a matrix $V$ whose columns are centered vectors $v_1, \dotsc, v_n$ such that $\sum_{j=1}^n v_j = 0$.  Let $J = I - \frac{1}{n} \mathbf{1} \mathbf{1}^\top $ be the centering matrix from MDS (\Cref{ALG:MDS}), $G = V^\top V$ be the Gram matrix for $V$, and $D$ be the squared distance matrix $D_{ij} = \Vert v_i - v_j \Vert^2$.  Then $G = -\frac{1}{2} J D J$.
\end{lemma}

\begin{proof}
Note first that
\begin{align*}
    (JDJ)_{ij} = D_{ij} + \frac{1}{n^{2}} \sum_{k, \ell = 1}^{n} D_{k\ell} - \frac1n \sum_{k=1}^n (D_{ik} + D_{kj}).
\end{align*}
Moreover, because $D_{ij} = v_i^\top v_i + v_j^\top v_j - 2 v_i^\top v_j$, we get that
\begin{align*}
    (JDJ)_{ij} &= v_i^\top v_i + v_j^\top v_j - 2 v_i^\top v_j + \frac{1}{n^{2}} \bigg( 2n \sum_{k=1}^n v_k^\top v_k \corr{-} 2 \mathbf{1}^\top V^\top V \mathbf{1} \bigg) \\
    &- \frac1n \bigg( n v_i^\top v_i + n v_j^\top v_j + 2 \sum_{k=1}^n v_k^\top v_k - 2 \mathbf{1}^\top V^\top v_j - 2 v_i^\top V \mathbf{1} \bigg).
\end{align*}
Note here that $V \mathbf{1} = 0$ since $\sum_{j=1}^n v_j = 0$.  After cancelling terms, we get
\begin{align*}
    (JDJ)_{ij} &= -2 v_i^\top v_j = - 2 G_{ij}.
\end{align*}
So our result is immediate.
\end{proof}

The next results are used to recount the $\eps$-compatibility as well as its effects on LOT.  First, we show that every $\eps$-compatible map has a compatible map (with $\eps=0$) nearby whose LOT distance from the $\eps$-compatible map is small. 

\begin{lemma}\label{LEM:almostCompat}
Assume that
\begin{enumerate}[(i)]
    \item $\sigma$ is supported on a compact convex set $\Omega \subset \R^n$ with probability density $f_\sigma$ bounded above and below by positive constants.
    
    \item $\mu$ has finite $p$-th moment with bound $M_p$ with $p > d$ and $p\geq 4$.

    \item There exist $A >0$ such that every $h \in \mathcal{H}$ satisfies $\Vert h (x) \Vert \leq A \Vert x \Vert$ for every $x \in \Omega$.
\end{enumerate}
Let $\mathcal{H}$ be $\eps$-compatible with respect to $\sigma$ and $\mu$.  Then for every $h \in \mathcal{H}$ there exists a compatible $g$ such that
\begin{align*}
    \Big\Vert T_\sigma^{g_\sharp \mu} - T_\sigma^{h_\sharp \mu} \Big\Vert_\sigma &\leq C_{n, p, \Omega, A^p M_p} \cdot \eps^{\frac{p}{6p + 16n } } \\
    \Vert h \circ T_\sigma^\mu - T_\sigma^{h_\sharp \mu } \Vert_\sigma &< \eps + C_{n,p,\Omega, A^p M_p} \cdot \eps^{\frac{p}{6p+16 n} }.
\end{align*}
\end{lemma}
\begin{proof}
Let $h \in \mathcal{H}$, then there exists an exactly compatible transformation $g$ such that $g \circ T_\sigma^\mu = T_\sigma^{g_\sharp \mu}$ with $\Vert h -g \Vert_\mu < \eps$ by definition of $\eps$-compatibility.  Then notice that
\begin{align*}
    \Big\Vert h \circ T_\sigma^\mu - T_\sigma^{h_\sharp \mu} \Big\Vert_\sigma &= \Big\Vert h \circ T_\sigma^\mu - g \circ T_\sigma^\mu + T_\sigma^{g_\sharp \mu} - T_\sigma^{h_\sharp \mu} \Big\Vert_\sigma \\
    &\leq \Vert h - g \Vert_\mu + \Big\Vert T_\sigma^{g_\sharp \mu} - T_\sigma^{h_\sharp \mu} \Big\Vert_\sigma .
\end{align*}
By assumption, we know that $\Vert h - g \Vert_\mu < \eps$.  Since $h \in \mathcal{H}$ satisfies (iii), we have
\begin{align*}
    \int_\Omega \Vert x\Vert^p f_{h_\sharp \mu}(x) dx &= \corr{\int_\Omega \Vert x\Vert^p d(h_\sharp \mu)(x) = \int_\Omega \underbrace{\Vert h(x) \Vert^p}_{\leq A^p \Vert x \Vert^p} d \mu(x) \leq A^p M_p}
\end{align*}
\corr{Similarly, note that $g \in \mathcal{H}$ because $\mathcal{H}$ consists of all $\eps$-compatible pushforwards with respect to $\sigma$ and $\mu$ (i.e. there exists compatible $g'$ such that $\Vert g - g' \Vert_\mu < \eps$ and $g' \circ T_\sigma^\mu = T_\sigma^{(g')_\sharp \mu}$), and we can see that $g$ is itself the compatible transformation.  This implies that we have the same moment bound for $g$.}  Now using \Cref{THM:merigotTHM} and equation 9 of \cite{ambrosio2013user}, we get that
\begin{align*}
    \Big\Vert T_\sigma^{g_\sharp \mu} - T_\sigma^{h_\sharp \mu} \Big\Vert_\sigma &\leq C_{n, p, \Omega, A^p M_p} W_1(g_\sharp \mu, h_\sharp \mu )^{\frac{p}{6p + 16n } } \\
    &\leq C_{n, p, \Omega, A^p M_p} W_2(g_\sharp \mu, h_\sharp \mu )^{\frac{p}{6p + 16n } } \\
    &\leq C_{n, p, \Omega, A^p M_p} \Vert h - g \Vert_\mu^{\frac{p}{6p + 16n}} \\ &  \leq C_{n, p, \Omega, A^p M_p} \cdot \eps^{\frac{p}{6p + 16n}}.
\end{align*}
This implies that
\begin{align*}
    \Vert h \circ T_\sigma^\mu - T_\sigma^{h_\sharp \mu} \Vert_\sigma < \eps + C_{n,p,\Omega, A^p M_p} \cdot \eps^{\frac{p}{6p+16 n}}.
\end{align*}
\end{proof}

Now we can show that the LOT embedding between exactly compatible transformations is isometric with the Wasserstein manifold.

\begin{lemma}\label{LEM:ExactCompat}
Let $g_1$ and $g_2$ be exactly compatible transformations, i.e. $g_1 \circ T_\sigma^\mu = T_\sigma^{(g_1)_\sharp \mu}$ and $g_2 \circ T_\sigma^\mu = T_\sigma^{(g_2)_\sharp \mu }$, then 
\begin{align*}
    \Big\Vert T_{\sigma}^{(g_1)_\sharp \mu} - T_\sigma^{(g_2)_\sharp \mu}  \Big\Vert_\sigma = W_2\Big( (g_1)_\sharp \mu, (g_2)_\sharp \mu \Big).
\end{align*}
\end{lemma}
\begin{proof}
First notice that since everything is absolutely continuous, we can use a change of variables formula to get
\begin{align*}
    \bigg\Vert T_{\sigma}^{(g_1)_\sharp \mu} - T_\sigma^{(g_2)_\sharp \mu}  \bigg\Vert_\sigma = \bigg\Vert I - T_\sigma^{(g_2)_\sharp \mu} \circ T_{(g_1)_\sharp \mu}^{\sigma} \bigg\Vert_{(g_1)_\sharp \mu}.
\end{align*}
Because $T_{(g_1)_\sharp \mu}^{(g_2)_\sharp \mu}$ is the minimizer of the optimal transport problem and the triangle inequality, we get
\begin{align*}
    W_2 \bigg( (g_1)_\sharp \mu, (g_2)_\sharp \mu \bigg) &= \bigg\Vert I - T_{(g_1)_\sharp \mu}^{(g_2)_\sharp \mu} \bigg\Vert_{(g_1)_\sharp \mu} \leq \bigg\Vert I - T_\sigma^{(g_2)_\sharp \mu} \circ T_{(g_1)_\sharp \mu}^{\sigma} \bigg\Vert_{(g_1)_\sharp \mu} \\
    &\leq \bigg\Vert I - T_{(g_1)_\sharp \mu}^{(g_2)_\sharp \mu} \bigg\Vert_{(g_1)_\sharp \mu} + \bigg\Vert T_{(g_1)_\sharp \mu}^{(g_2)_\sharp \mu} - T_\sigma^{(g_2)_\sharp \mu} \circ T_{(g_1)_\sharp \mu}^{\sigma} \bigg\Vert_{(g_1)_\sharp \mu}.
\end{align*}
Note that Theorem 24 of \cite{khurana2022supervised} implies that given an exactly compatible transformation $g$, $J_g(T_\sigma^\mu(x))$ must share the same eigenspaces as $J_{T_\sigma^\mu}(x)$.  
By Corollary 4 of \cite{khurana2022supervised}, we know that exactly compatible transformations are optimal transport maps themselves.  This means that $T_\mu^{g_\sharp \mu} = g$ for exactly compatible transport maps.  Moreover, for an exactly compatible $h' \in \mathcal{H}$, this means that $T_{ g_\sharp \mu}^{ (g')_\sharp \mu} = g' \circ g^{-1}$ because $g' \circ g^{-1}$ is a gradient of a convex function (since the Jacobian of $g$ and $g'$ share the same eigenspaces) that pushes $g_\sharp \mu$ to $(g')_\sharp \mu$.  In the context of $g_1$ and $g_2$, this gives us that
\begin{align*}
    \corr{T_{(g_1)_\sharp \mu}^{(g_2)_\sharp \mu} = g_2 \circ g_1^{-1} = g_2 \circ T_\sigma^\mu \circ T_\mu^\sigma \circ g_1^{-1} = T_\sigma^{(g_2)_\sharp \mu} \circ T_{(g_1)_\sharp \mu}^{\sigma}.}
\end{align*}
In particular, we get that
\begin{align*}
    \Big\Vert T_{\sigma}^{(g_1)_\sharp \mu} - T_\sigma^{(g_2)_\sharp \mu}  \Big\Vert_\sigma = W_2\Big( (g_1)_\sharp \mu, (g_2)_\sharp \mu \Big).
\end{align*}
\end{proof}

Finally, we show that $\eps$-compatible transformations have LOT embeddings that are ``$\eps^{\frac{p}{6p+16n}}$-isometric" in the sense of the following theorem. 

\begin{theorem}\label{THM:epsCompat}
Assume that
\begin{enumerate}[(i)]
    \item $\sigma$ is supported on a compact convex set $\Omega \subset \R^n$ with probability density $f_\sigma$ bounded above and below by positive constants.
    
    \item $\mu$ has finite $p$-th moment with bound $M_p$ with $p > n$ and $p\geq 4$.

    \item There exists constants $a,A >0$ such that every $h \in \mathcal{H}$ satisfies $a \Vert x \Vert \leq \Vert h (x) \Vert \leq A \Vert x \Vert$.
\end{enumerate}
Let $\mathcal{H}$ be $\eps$-compatible with respect to absolutely continuous measures $\sigma$ and $\mu$ and assume that $h_\sharp \mu$ is absolutely continuous.  Then for $h_1, h_2 \in \mathcal{H}$,
\begin{align*}
    \Bigg\vert W_2 \bigg( (h_1)_\sharp \mu, (h_2)_\sharp \mu \bigg) - \bigg\Vert T_{\sigma}^{(h_1)_\sharp \mu} - T_\sigma^{(h_2)_\sharp \mu}  \bigg\Vert_\sigma \Bigg\vert < 2\Big( \eps + C_{n, p, \Omega, a^{-1} A^p M_p} \cdot \eps^{ \frac{p}{ 6p + 16n } } \Big) < C \eps^{ \frac{p}{6p + 16n} }
\end{align*}
\end{theorem}
\begin{proof}
By definition, we know that there exist $g_1$ and $g_2$ such that $\Vert g_1 - h_1 \Vert_\mu < \eps$ and $\Vert g_2 - h_2 \Vert_\mu < \eps$.
First, note that 
\begin{align*}
    \Big\Vert T_{\sigma}^{(h_1)_\sharp \mu} - T_\sigma^{(h_2)_\sharp \mu}  \Big\Vert_\sigma \leq \Big\Vert T_{\sigma}^{(h_1)_\sharp \mu} - T_\sigma^{(g_1)_\sharp \mu}  \Big\Vert_\sigma + \Big\Vert T_{\sigma}^{(g_1)_\sharp \mu} - T_\sigma^{(g_2)_\sharp \mu}  \Big\Vert_\sigma + \Big\Vert T_{\sigma}^{(g_2)_\sharp \mu} - T_\sigma^{(h_2)_\sharp \mu}  \Big\Vert_\sigma.
\end{align*}
By \Cref{LEM:ExactCompat}, we know that
\begin{align*}
    \Big\Vert T_{\sigma}^{(g_1)_\sharp \mu} - T_\sigma^{(g_2)_\sharp \mu}  \Big\Vert_\sigma &= W_2 \Big( (g_1)_\sharp \mu, (g_2)_\sharp \mu \Big).
\end{align*}
However,  by equation 2.1 of \cite{ambrosio2013user} and the triangle inequality, we have
\begin{align*}
    W_2 \Big( (g_1)_\sharp \mu, (g_2)_\sharp \mu \Big) &\leq \underbrace{ W_2 \Big( (g_1)_\sharp \mu, (h_1)_\sharp \mu \Big)}_{\leq \Vert g_1 - h_1 \Vert_\mu < \eps} + W_2 \Big( (h_1)_\sharp \mu, (h_2)_\sharp \mu \Big) + \underbrace{ W_2 \Big( (h_2)_\sharp \mu, (g_2)_\sharp \mu \Big) }_{ \leq \Vert h_2 - g_2 \Vert_\mu < \eps } \\
    &\leq W_2 \Big( (h_1)_\sharp \mu, (h_2)_\sharp \mu \Big) + 2 \eps .
\end{align*}
Moreover, by \Cref{LEM:almostCompat}, for $i = 1,2$, we know that
\begin{align*}
    \Big\Vert T_{\sigma}^{(g_i)_\sharp \mu} - T_\sigma^{(h_i)_\sharp \mu}  \Big\Vert_\sigma \leq C_{n,p, \Omega, a^{-1} A^p M_p} \cdot \eps^{ \frac{p}{6p + 16n } }.
\end{align*}
This implies that
\begin{align*}
    W_2 \Big( (h_1)_\sharp \mu, (h_2)_\sharp \mu \Big)  &\leq \Big\Vert T_{\sigma}^{(h_1)_\sharp \mu} - T_\sigma^{(h_2)_\sharp \mu}  \Big\Vert_\sigma \\ & \leq W_2 \Big( (h_1)_\sharp \mu, (h_2)_\sharp \mu \Big) + 2\Big( \eps + C_{n, p, \Omega, a^{-1} A^p M_p} \eps^{ \frac{p}{ 6p + 16n } } \Big),
\end{align*}
and the proof is complete.
\end{proof}

\section{Plug-in estimator approximation results}\label{ap:estimatorapprox}

In this section, we provide some auxiliary results that are used along the way to prove the theorems of \Cref{sec:compact_support}.

\subsection{Using the Linear Program to compute transport maps}
Recall that for a random variable $X_m$, we say that $X_m = O_p(a_m)$ if for every $\varepsilon > 0$ there exists $M > 0$ and $N > 0$ such that
\begin{align*}
    \mathbb{P}\Big( \vert X_m / a_m \vert > M \Big) < \eps \quad \forall m \geq N.
\end{align*}

The following theorem from \cite{deb2021rates} is used in the proofs of our main results, including \Cref{THM:CptLipshitzLOT}.

\begin{theorem}[{\cite[Theorem 2.2]{deb2021rates}}]\label{debrate}
Suppose that $T_\sigma^{\mu}$ is $L$-Lipschitz, and $\mu$ is compactly supported and $\E_{\sigma}[ \exp(t \Vert x \Vert^\alpha ) ] < \infty$ for some $t > 0, \alpha > 0$. 
 Assume we draw $k$ i.i.d. samples from $\mu$ and consider the estimator $\widehat{\mu}$.  Then
\begin{align*}
    \sup_{\gamma\in \Gamma_{\min}} \int \Vert T_\sigma^{\widehat{\mu}}(x; \gamma_{LP}) - T_\sigma^{\mu}(x) \Vert^2 d\sigma(x) \le O_p(r_n^{(k)}\log(1+k)^{t_{n,\alpha}}),
\end{align*}
where 
\begin{align*}
    r_n^{(k)} = \begin{cases}
                    2k^{-1/2} & n = 2,3 \\
                    2 k^{-1/2} \log(1 + k) & n = 4 \\
                    2 k^{-2/n} & n \geq 5
                \end{cases}, \hspace{0.3cm} t_{n,\alpha} = \begin{cases}
                    (4\alpha)^{-1}(4 + (( 2 \alpha + 2 n \alpha - n) \lor 0)) & n < 4 \\
                    (\alpha^{-1} \lor 7/2) - 1 & n = 4 \\
                    2(1 + n^{-1}) & n > 4 
                \end{cases},
\end{align*}
so that $r_n^{(k)}$ and $t_{n,\alpha}$ are on the order of $k^{-1/n}$ and $2(1+n^{-1})$, respectively.
\end{theorem}

\corr{We utilize the above theorem for the case that $\sigma$ is compactly supported in the proof of \Cref{THM:CptLipshitzLOT}. Consequently, the expectation bound holds for all $t,\alpha>0$. We see that for $n<4$, one can choose $\alpha$ large enough so that $t_{n,\alpha}$ is arbitrarily close to $1$ if $n=1$, $\frac32$ if $n=2$, and $2$ if $n=3$. Similarly, for $n=4$ we may choose $\alpha$ large enough so that  $t_{n,\alpha}=\frac52$. This simplifies the statement in \Cref{THM:CptLipshitzLOT}.}

\begin{remark}
We note that \Cref{debrate} is the ``semi-discrete'' version described in \cite{deb2021rates}. The paper also provides equivalent bounds in the instance that $\sigma$ is similarly estimated.  However, the bounds only guarantee that the transport maps agree when integrated against $\widehat{\sigma}$, whereas we need the bound for $\sigma$ itself.  
\end{remark}

\subsection{Approximating with Finite Samples from the Reference Distribution}

Some of the norms from \Cref{THM:CptLipshitzLOT} and \Cref{THM:sinkEstThm} are assumed to be integrated against the true $\sigma$.  However, we need to consider the discretized $\sigma$ for each norm, and establish that we can estimate these norms with high probability.   For these bounds, we use McDiarmid's inequality on the function
\begin{align*}
    f(X_1, ..., X_m) = \frac{1}{m} \sum_{j=1}^m \big\vert T_\sigma^{\widehat{\mu_1}}(X_j; \gamma_{\widehat{\mu_1}}) - T_\sigma^{\widehat{\mu_2}}(X_j; \gamma_{\widehat{\mu_2}}) \big\vert^2 = \WLOTest(\widehat{\mu_1},\widehat{\mu_2}; \gamma)^2,
\end{align*}
where $X_j\sim \sigma$, $\gamma_{\widehat{\mu_j}}$ is a transport plan between $\sigma$ and $\widehat{\mu_j}$ for $j=1,2$, and $\gamma \in \{\gamma_{LP}, \gamma_\beta \}$ denotes the optimization method used to get $\gamma_{\widehat{\mu_j}}$.  If $\mu_i$ are supported in a ball of radius $R$, then  McDiarmid's inequality implies
\[
    \Prob\left( \left| \frac{1}{m} \sum_{j=1}^m |T_\sigma^{\widehat{\mu_1}}(X_j; \gamma_{\widehat{\mu_1}}) - T_\sigma^{\widehat{\mu_2}}(X_j; \gamma_{\widehat{\mu_2}})|^2 - \| T_\sigma^{\widehat{\mu_1}}(\cdot ; \gamma_{\widehat{\mu_1}}) - T_\sigma^{\widehat{\mu_2}}(\cdot ; \gamma_{\widehat{\mu_2}}) \|_{\sigma}^2  \right| > t \right) \le 2e^{-m \frac{t^2}{32R^4}}.
\]
Note that since $f = \WLOTest(\widehat{\mu_1},\widehat{\mu_2}; \gamma)^2$, we get
\begin{align}\label{eq:Mcdiarmid}
    \Prob\left( \left|\WLOTest(\widehat{\mu_1},\widehat{\mu_2}; \gamma )^2 -\WLOT(\widehat{\mu_1},\widehat{\mu_2}; \gamma )^2  \right| > t \right) &\le 2e^{-m \frac{t^2}{32R^4}}.
\end{align}

\begin{theorem} \label{thm:mcdiarmid}
Consider $\mu_i,\sigma \in W_2(\R^n)$. Assume supp$(\mu_i)\subset B(0,R)$ for $i=1,2$. Let $\delta>0$. Then with probability at least $1 - \delta$,
    \begin{align*}
        \left|  \WLOT(\widehat{\mu_1}, \widehat{\mu_2}; \gamma) - \WLOTest(\widehat{\mu_1}, \widehat{\mu_2}; \gamma) \right| \leq R \sqrt{  \frac{ 2 \log( 2 / \delta )  }{ m }  },
    \end{align*}
    where $m$ is the number of samples used to estimate $\sigma$.  \corr{Here, if $\sigma$ is not absolutely continuous, then the transport maps are constructed from barycentric projections of transport plans.}
\end{theorem}

\begin{proof}
Define
\begin{align*}
    a = \WLOT(\widehat{\mu_1}, \widehat{\mu_2}; \gamma), \quad b = \WLOTest(\widehat{\mu_1}, \widehat{\mu_2}; \gamma).
\end{align*}
Then both $a\leq 2R$ and $b\leq 2R$.
Now, since $a^2 - b^2 = (a+b)(a-b)$, we get that
\[|a-b| \geq \frac{1}{4R}|a^2-b^2|.\]
This, together with \eqref{eq:Mcdiarmid}, implies that
\begin{align*}
    \Prob\left(  \left|\WLOTest(\widehat{\mu_1},\widehat{\mu_2}; \gamma)-\WLOT(\widehat{\mu_1},\widehat{\mu_2}; \gamma)  \right| > t \right) &\le 2e^{-m \frac{t^2}{2R^2}}.
\end{align*}
Solving $\delta = 2e^{-m \frac{t^2}{2R^2}}$ for $t$ yields the conclusion.

\end{proof}

\corr{The following corollary is geared towards showing \Cref{THM:sinkEstThm} as we use the estimated optimal transport map generated from using the Sinkhorn transport plan solution.}

\begin{corollary}\label{cor:mcdiarmidSink}
    \corr{Under the assumptions of \Cref{thm:mcdiarmid}, suppose $X_j \sim \sigma$ i.i.d. for ($j=1,\dotsc,m$) and let $\widehat{\sigma} = \frac{1}{m} \sum_{j=1}^{m} \delta_{X_j}$. Then with probability at least $1 - \delta$,
    \begin{align*}
        \Big\vert \| T_{\widehat{\sigma}}^{\widehat{\mu_1}}(\ \cdot\ ; \gamma_{\widehat{\mu_1},\beta } ) - T_{\widehat{\sigma}}^{\widehat{\mu_2}}(\ \cdot\ ; \gamma_{\widehat{\mu_2},\beta } ) \|_{\sigma}^2 - \widehat{W}_{2,\widehat{\sigma}}^{LOT}(\widehat{\mu_i}, \widehat{\mu_j}) \Big\vert \leq R \sqrt{ \frac{2 \log(2/\delta) }{ m } }.
    \end{align*}
   }
\end{corollary}
\begin{proof}
    \corr{Use the sampling $X_j \sim \sigma$ that generates $\widehat{\sigma}$, we can use McDiarmid's inequality to get
    \begin{align*}
        \Prob\left( \left| \frac{1}{m} \sum_{j=1}^m |T_{\widehat{\sigma}}^{\widehat{\mu_1}}(X_j; \gamma_{\widehat{\mu_1}, \beta}) - T_{\widehat{\sigma}}^{\widehat{\mu_2}}(X_j; \gamma_{\widehat{\mu_2},\beta } )|^2 - \| T_{\widehat{\sigma}}^{\widehat{\mu_1}}(\ \cdot\ ; \gamma_{\widehat{\mu_1},\beta } ) - T_{\widehat{\sigma}}^{\widehat{\mu_2}}(\ \cdot\ ; \gamma_{\widehat{\mu_2},\beta } ) \|_{\sigma}^2  \right| > t \right) \le 2e^{-m \frac{t^2}{32R^4}}.
    \end{align*}
    We use the same sampling for $\widehat{\sigma}$ as we do for the concentration for McDiarmid's inequality.  Notice, however, that 
    \begin{align*}
        \underbrace{\frac{1}{m} \sum_{j=1}^m |T_{\widehat{\sigma}}^{\widehat{\mu_1}}(X_j; \gamma_{\widehat{\mu_1}, \beta}) - T_{\widehat{\sigma}}^{\widehat{\mu_2}}(X_j; \gamma_{\widehat{\mu_2},\beta } )|^2}_{\widehat{W}_{2,\widehat{\sigma}}^{LOT}(\widehat{\mu_i}, \widehat{\mu_j})} &= \underbrace{ \| T_{\widehat{\sigma}}^{\widehat{\mu_1}}(\ \cdot\ ; \gamma_{\widehat{\mu_1},\beta } ) - T_{\widehat{\sigma}}^{\widehat{\mu_2}}(\ \cdot\ ; \gamma_{\widehat{\mu_2},\beta } ) \|_{\widehat{\sigma}}^2}_{W_{2,\widehat{\sigma}}^{LOT}(\widehat{\mu_i},\widehat{\mu_j}) }.
    \end{align*}
    This yields the result.}
\end{proof}

\section{Non-Compactly Supported Measures Proofs and Results}\label{ap:noncompactProofs}

Here, we give the proofs of the lemmas preceding \Cref{THM:BaryNonCompact,THM:sinkNonCompact}. 

\begin{proof}[Proof of \Cref{LEM:extLEMM1}]\label{PF:extLEMM1}
    We will construct the measure $\widetilde{\mu}$ by constructing a transport map that sends $\mu$ to a compactly supported absolutely continuous measure.  In particular, for some $0 < \rho \ll 1$, consider the map
    \begin{align*}
        S_{R,\rho} (x) = \begin{cases}
            x & x \in B(0,R) \\
            R \frac{x}{\Vert x \Vert} + \min\{ \Vert x \Vert - R, \rho \} \frac{x}{1 + \Vert x \Vert } & x \not\in B(0,R)
        \end{cases}.
    \end{align*}
    Let $\widetilde{\mu} = (S_{R,\rho})_{\sharp} \mu$, then the compact set that $\widetilde{\mu}$ will be supported on is \corr{$\overline{B(0,R+\rho)}$ since for $x \in \R^n$ with $\Vert x \Vert \gg R$ we have $\Vert S_{R,\rho}(x) \Vert < R + \rho$}.  Now note that
    \begin{align*}
        W_1(\mu,\widetilde{\mu}) &= \min_{S: S_\sharp \mu = \widetilde{\mu}} \int_{\mathbb{R}^n} \Vert S(x) - x \Vert d\mu(x) \leq \int_{\mathbb{R}^n} \Vert S_{R,\rho}(x) - x \Vert d \mu(x) \\
        &= \int_{B(0,R)} \underbrace{ \Vert x - x \Vert}_{= 0} d\mu(x) + \int_{\mathbb{R}^n \setminus B(0,R)} \bigg\Vert \bigg( 1 - \frac{R}{\Vert x \Vert}  - \frac{\min\{\Vert x \Vert - R , \rho \} }{1 + \Vert x \Vert} \bigg) x \bigg\Vert d\mu(x) \\
        &\leq  \int_{\mathbb{R}^n \setminus B(0,R)} \Vert x \Vert + \underbrace{R}_{\leq \Vert x \Vert} + \underbrace{\frac{\Vert x \Vert \min\{\Vert x \Vert - R , \rho \} }{1 + \Vert x \Vert}}_{\leq \rho \leq 1 \leq \Vert x \Vert} d\mu(x) \leq \int_{\mathbb{R}^n \setminus B(0,R)} 3 \Vert x \Vert d\mu(x) .
    \end{align*}
    \corr{To bound this integral, recall that 
    \begin{align*}
        f_\mu(x) < \bigg( \frac{\eta}{\widetilde{C}_{n, p, \Omega, M_p}} \bigg)^{\frac{6p+16n}{p}} \frac{1}{ \Vert x \Vert^{n+2} }
    \end{align*}
    with 
    \begin{align*}
        \widetilde{C}_{n, p, \Omega, M_p} = C_{n, p, \Omega, M_p} (3C)^{\frac{p}{6p+16n}}
    \end{align*}
    where $C$ denotes a constant of integration over concentric $n$-spheres.}  Recall that $d\mu(x) = f_\mu(x) dx$; thus,
    \begin{align*}
        \int_{\mathbb{R}^n \setminus B(0,R)} 3 \Vert x \Vert d\mu(x) &= \int_{\mathbb{R}^n \setminus B(0,R)} 3 \Vert x\Vert f_\mu(x) dx \\
        &\leq \int_{\mathbb{R}^n \setminus B(0,R)}  \bigg( \frac{\eta}{C_{n, p, \Omega, M_p}} \bigg)^{\frac{6p+16n}{p}} \frac{1}{ C \Vert x \Vert^{n+1} } dx \\
        &\leq  \bigg( \frac{\eta}{C_{n, p, \Omega, M_p}} \bigg)^{\frac{6p+16n}{p}} \underbrace{\int_{ r \geq R} \frac{r^{n-1}}{r^{n+1}} dr }_{\leq 1} \\
        &= \bigg( \frac{\eta}{C_{n, p, \Omega, M_p}} \bigg)^{\frac{6p+16n}{p}} ,
    \end{align*}
    where $C$ is a constant from integrating over concentric $n$-spheres.  Invoking \Cref{THM:merigotTHM}, this means that
    \begin{align*}
        \Vert T_\sigma^\mu - T_\sigma^{\widetilde{\mu}} \Vert_{\sigma} \leq C_{n, p, \Omega, M_p} W_1(\mu,\widetilde{\mu})^{\frac{p}{6p + 16n}} \leq C_{n, p, \Omega, M_p}  \frac{\eta}{C_{n, p, \Omega, M_p}}  = \eta.
    \end{align*}
    To see that $\widetilde{\mu}$ is compactly supported, notice that for $x \in \mathbb{R}^n \setminus B(0,R)$, we have
    \begin{align*}
        \Vert S_{R,\rho}(x) \Vert = \bigg\Vert R \frac{x}{\Vert x \Vert} + \min\{ \Vert x \Vert - R, \rho \} \frac{x}{1 + \Vert x \Vert } \bigg\Vert \leq R + \rho \underbrace{\frac{\Vert x \Vert }{1 + \Vert x \Vert}}_{\leq 1} \leq R + \rho.
    \end{align*}
    The case for when $x \in B(0,R)$ is trivial since $S_{R,\rho}$ is the identity map on $B(0,R)$.  Moreover, to see that $\widetilde{\mu}$ is absolutely continuous with respect to the Lebesgue measure, we will take a generic set $A$ and break it up into components and analyze each component.  We first notice that $S_{R,\rho}$ is continuous.  Indeed, for $x$ such that $\Vert x \Vert = R$, we see that
    \begin{align*}
        \underbrace{R \frac{x}{\Vert x \Vert}}_{x} + \underbrace{\min\{ \Vert x \Vert - R, \rho\}}_{= \Vert x \Vert - R = 0} \frac{x }{1 + \Vert x \Vert} = x.
    \end{align*}
    Now, let $A \in \mathbb{R}^n$ such that $\lambda(A) = 0$ for the Lebesgue measure $\lambda$, then
    \begin{align*}
        A &= (A \cap B(0,R) ) \oplus (A \setminus \overline{B(0,R)} ) \oplus ( A \cap \partial B(0,R)) \\
        \implies (S_{R,\rho})_\sharp \mu( A) &= (S_{R,\rho})_\sharp \mu ( A \cap B(0,R) ) + (S_{R,\rho})_\sharp \mu ( A \setminus \overline{B(0,R)} ) + (S_{R,\rho})_\sharp \mu ( A \cap \partial B(0,R) ) \\
        &= \mu ( {S_{R,\rho}}^{-1} (A \cap B(0,R) ))  + \mu ({S_{R,\rho}}^{-1}( A \setminus \overline{B(0,R)} )) + \mu ({S_{R,\rho}}^{-1} ( A \cap \partial B(0,R) )) \\
        &= \mu( A \cap B(0,R)) + \underbrace{\mu( A \cap \partial B(0,R) )}_{\leq \mu(\partial B(0,R)) = 0} + \mu( {S_{R,\rho}}^{-1} (A \setminus \overline{B(0,R)} )),
    \end{align*}
    where we use the additivity of measures over disjoint sets, the form of $S_{R,\rho}$ on $B(0,R)$, and the absolutely continuity of $\mu$ so that $\mu( \partial B(0,R)) \leq \lambda( \partial B(0,R)) = 0$.  Moreover, note that $\mu(A \cap B(0,R)) \leq \mu(A) \leq \lambda(A) = 0$.  The only term left is $A \setminus \overline{B(0,R)}$.  Since $S_{R,\rho}$ is smooth on $\mathbb{R}^n \setminus B(0,R)$, there exists a density $g$ for $(S_{R,\rho})_\sharp \mu$ with respect to $\mu$ for sets in $\mathbb{R}^n \setminus B(0,R)$.  This means $({S_{R,\rho}})_\sharp \mu  \ll \mu$ on $\mathbb{R}^n \setminus \overline{B(0,R)}$.  Since $\mu \ll \lambda$, we have
    \begin{align*}
        \lambda(A) = 0 \implies \mu(A) = 0 \implies \mu( A \setminus \overline{B(0,R)} ) = 0 \implies (S_{R,\rho})_\sharp \mu( A \setminus \overline{B(0,R)} ) = 0.
    \end{align*}
    This shows that $(S_{R,\rho})_\sharp \mu$ is absolutely continuous with respect to $\lambda$, so the proof is complete.
\end{proof}

\begin{proof}[Proof of \Cref{LEM:extLEMM2}]\label{PF:extLEMM2}
    Rather than constructing a transport map, we will construct a density $f_{\widetilde{\mu}}$ and will argue that the transport map from $\mu$ to $\widetilde{\mu}$ (the measure with density $f_{\widetilde{\mu}}$) behaves nicely.  To do this, consider the following density
    \begin{align*}
        f_{\widetilde{\mu}, a, R}(x) = \begin{cases}
            f_\mu(x) & x \in B(0,R) \\
            f_\mu\Big( R \frac{x}{\Vert x \Vert} \Big) + \alpha \Big( \frac{\Vert x \Vert}{R} - 1 \Big) & x \in B(0,a) \setminus B(0,R) \\
            0 & \text{otherwise}
        \end{cases},
    \end{align*}
    for some $\alpha > 0$.  Notice that $a$ is not specified at the moment, but it depends on $R$ and $\alpha$.  Since we want $\widetilde{\mu}$ to be a probability measure, we note that
    \begin{align*}
        \widetilde{\mu}( \R^d ) = \underbrace{\int_{B(0,R)} f_\mu(x) dx }_{\mu(B(0,R))} + \underbrace{\int_{R}^{a} r^{d-1} C(r) \bigg( f_\mu\Big( R \frac{x}{\Vert x \Vert} \Big) + \alpha \Big( \frac{\Vert x \Vert}{R} - 1 \Big)\bigg) dr}_{ I(a) },
    \end{align*}
    where $C(r)$ is the integral over the sphere at radius $r$.  Notice that $I(a)$ has an integrand that is increasing as a function of $r$ so that $I(a)$ itself is increasing as a function of $a$ (i.e. $\lim_{a \to \infty} I(a) = \infty$).  Moreover, because $I(R) = 0$, we know from the intermediate value theorem that there exists some $a^*$ such that $I(a^*) = \mu( \R^d \setminus B(0,R))$.  Note that from this construction, $\widetilde{\mu}$ is compactly supported, absolutely continuous with respect to the Lebesgue measure, and $0 < c \leq b \leq f_{\widetilde{\mu}} \leq B < \infty$ for some constants $b$ and $B$.
    
    Now, we would like to bound $W_1(\mu, \widetilde{\mu})$.  \corr{Because we assume that $\mu$ has a density, there exists pushforwards that push $\mu$ to $\widetilde{\mu}$, but} we will consider $S$ such that $S_\sharp \mu = \widetilde{\mu}$ and $S(x) = x$ if $x \in B(0,R)$.  \corr{The set of such maps $S$ is non-empty, because we can consider the optimal transport problem between the restricted measures $\mu\vert_{\R^d \setminus B(0,R)}$ to $\widetilde{\mu}\vert_{\R^d \setminus B(0,R)}$ with the same total mass.  There certainly exist pushforwards, say $S_{\vert}$ between the restricted measures because the measures have densities.  To form a map $S$, let $S$ be the identity on $B(0,R)$ and $S_\vert$ on $\R^d \setminus B(0,R)$.} Note that $S(x) \in B(0,a)$ for $x \in B(0,a) \setminus B(0,R)$; thus, there exists $\widetilde{C}$ such that $\Vert S(x) \Vert \leq \widetilde{C} \Vert x \Vert$ (if $a < 2R$, then $\widetilde{C} \leq 2$).  For the following calculation, we assume \corr{that 
    \begin{align*}
        f_\mu(x) &\leq \bigg( \frac{\eta}{\widetilde{C}_{n,p,\Omega, M_p}' } \bigg)^{\frac{6p+16n}{p}} \frac{1}{  \Vert x \Vert^{n+2} }\\
        &:= \Bigg( \frac{\eta}{C_{n,p,\Omega, M_p} \big( (\Tilde{C}+ 1) C_{\text{sphere}} \big)^{\frac{p}{6p+16n} } } \Bigg)^{\frac{6p+16n}{p}} \frac{1}{ \Vert x \Vert^{n+2} } \\
        &= \bigg( \frac{\eta}{C_{n,p,\Omega, M_p} } \bigg)^{\frac{6p+16n}{p}}  \frac{1}{ (\widetilde{C} + 1) C_{\text{sphere}}  \Vert x \Vert^{n+2} } ,
    \end{align*}
    where $C_{\text{sphere}}$ denotes a constant from integrating over concentric $n$-spheres and $C_{n,p,\Omega,M_p}$ denotes the constant from \Cref{THM:merigotTHM}.} Now note that
    \begin{align*}
        W_1(\mu,\widetilde{\mu}) &\leq \int_{\R^d} \Vert S(x) - x \Vert d\mu(x) = \int_{B(0,R)} \underbrace{ \Vert x - x \Vert}_{= 0} d\mu(x) + \int_{\mathbb{R}^d \setminus B(0,R)} \Vert S(x) - x \Vert d\mu(x) \\
        &\leq \int_{\mathbb{R}^d \setminus B(0,R)} \Vert S(x)\Vert  + \Vert x \Vert d\mu(x) \leq \int_{\R^d \setminus B(0,R)} (\widetilde{C} + 1) \Vert x \Vert f_\mu(x) dx \\
        &\leq \int_{\mathbb{R}^d \setminus B(0,R)}  (\widetilde{C} + 1) \bigg( \frac{\eta}{C_{n, p, \Omega, M_p}} \bigg)^{\frac{6p+16n}{p}} \frac{1}{ ( \widetilde{C} + 1)  C_{\text{sphere}} \Vert x \Vert^{n+1} } dx \\
        &\leq  \bigg( \frac{\eta}{C_{n, p, \Omega, M_p}} \bigg)^{\frac{6p+16n}{p}} \underbrace{\int_{ r \geq R} \frac{r^{n-1}}{r^{n+1}} dr }_{\leq 1} \leq \bigg( \frac{\eta}{C_{n, p, \Omega, M_p}} \bigg)^{\frac{6p+16n}{p}}.
    \end{align*}
    Invoking \Cref{THM:merigotTHM}, this means that
    \begin{align*}
        \Vert T_\sigma^\mu - T_\sigma^{\widetilde{\mu}} \Vert_{\sigma} \leq C_{n, p, \Omega, M_p} W_1(\mu,\widetilde{\mu})^{\frac{p}{6p + 16n}} \leq C_{n, p, \Omega, M_p}  \frac{\eta}{C_{n, p, \Omega, M_p}}  = \eta.
    \end{align*}
    Thus, we have the desired result.
\end{proof}

\section{Proofs and Results for Conditions on $\mathcal{H}$ and $\mu$}\label{ap:conditions}

This section provides the proofs of the results in \Cref{SEC:CondCpt,SEC:CondNonCpt}.

\subsection{Compact Case Proofs and Results}

Here we prove the results of \Cref{SEC:CondCpt} which provide conditions on $\sigma$, $\mu$, and $\mathcal{H}$ which guarantee that $\mu_i\sim\mathcal{H}_\sharp \mu$ satisfy the conditions of the theorems from \Cref{sec:compact_support}.

\begin{proof}[Proof of \Cref{THM:baryCMPTCond_NonCompat}]\label{pf:baryCMPTCond_NonCompat}
    \corr{Caffarelli's regularity theorem implies that $T_\sigma^{(h_i)_\sharp \mu}$ is continuous, hence Lipschitz (since $\text{supp}(\mu)$ is compact). To show that Cafferelli's theorem applies, $(h_i)_\sharp \mu$ needs convex support and its density needs to be bounded away from 0 and $\infty$.  One of the assumptions in this theorem is that $(h_i)_\sharp \mu$ has convex support; thus we must show the density is bounded away from $0$ and $\infty$. If $h_i$ is continuously differentiable with $0 I \prec L_1 I \preceq J_{h_i}(x) $, then the minimum eigenvalue of $J_{h_i}$ is bounded away from $0$.  Noticing that $h_i$ is a proper map, we can use Hadamard's global inverse function theorem to see that $h_i^{-1}$ exists. Recalling $\mu$ has density $f_\mu$, we can use the change of variables density formula
    \begin{align*}
        f_{(h_i)_\sharp \mu }(x) = f_\mu( h_i^{-1}(x) ) \vert J_{h_i^{-1}}(x) \vert.
    \end{align*}
    Since $\mu$ has a density such that $0 < c \leq f_\mu \leq C < \infty$, we can see that $f_{(h_i)_\sharp \mu}$ is bounded away from 0 and $\infty$ by using the change of variables density formula.  In particular, note that $f_{(h_i)_\sharp \mu} \geq c L_1 > 0$.  For the upper-bound, we need that $h_i$ is Lipschitz.  This is immediate since $h_i$ is a continuously differentiable (hence continuous) map over a compact set; thus, $h_i$ is Lipschitz with some Lipschitz constant $L_2$.  This means that $f_{(h_i)_\sharp \mu}(x) \leq C L_2 < \infty$.}
\end{proof}

\begin{proof}[Proof of \Cref{THM:baryCMPTCond}]\label{PF:baryCMPTCond}
\corr{For the barycentric map estimator, we already showed that the $\mu_i$'s are compactly supported in a ball of radius $AR$ in the proof of \Cref{THM:CptLipshitzLOT}.  We now show that $T_\sigma^{\mu_i}$ is Lipschitz.  To make sure that each $T_\sigma^{\mu_i}$ is Lipschitz, we will only need that $h_i$ is continuous because continuous maps over compact sets are Lipschitz.  In particular, note that $\mu_i = (h_i)_\sharp \mu$ for some $h_i \in \mathcal{H}$.  Moreover, notice that we only care about how $h_i$ acts on $\text{supp}(\mu)$, which is compact; thus, $h_i$ is Lipschitz on the set of interest.  Now, by compatibility, we know that $T_\sigma^{\mu_i} = h_i \circ T_\sigma^\mu$, which implies that if $h_i$ is Lipschitz and $T_{\sigma}^\mu$ is Lipschitz, then $T_{\sigma}^{\mu_i}$ is Lipschitz.}


\end{proof}

\begin{proof}[Proof of \Cref{THM:sinkCMPTCond}]\label{PF:sinkCMPTCond}
For the entropic map estimator, the $\mu_i$'s need to again be compactly supported, $T_\sigma^{\mu_i}$ needs to be Lipschitz, and $\sigma$ and $\mu_i$ together satisfy assumptions $(A1)-(A3)$.  \corr{It will turn out, that we will only need that there exist constants $L_1, L_2 > 0$ such that
\begin{align*}
    L_1 I \preceq J_{h}(x) \preceq L_2 I.
\end{align*}
This occurs if $h$ is continuous, compatible, and has lower bound on Jacobian as $L_1 I \preceq J_h(x)$.  To see the upper bound, we just see that $h$ being a continuous map on a compact set (support of $\mu$) gives that on the support of $\mu$, we have $J_h(x) \preceq L_2 I$ for some $L_2$.  Now similar to \Cref{rem:assumptions-H-for-Lipschitz}, we can use Hadamard's global inverse function theorem to see that $h^{-1}$ exists with $L_2^{-1} I \preceq J_{h^{-1}}(x) \preceq L_1^{-1} I$.  Since we only sample a finite number of measure-valued data points $\mu_i$, we know that there must be a maximum $L_2$ that applies for all $\mu_i$.}

That $\mu_i$ is compactly supported and each $T_\sigma^{\mu_i}$ are Lipschitz follow from the same analysis as in the proof of \Cref{THM:baryCMPTCond}.

 \begin{itemize}
    

    \item \textbf{Ensuring that $\mu_i$ satisfy $(A1)$:}  Recall that the change of variables formula for the density of a pushforward measure $\widetilde{\mu} = h_\sharp \mu$ is given by
    \[
        f_{\widetilde{\mu}}(x) = f_{\mu}( h^{-1}(x)) \vert J_{h^{-1}}(x) \vert,
    \]
    where $\vert J_{h^{-1}}(x) \vert$ denotes the determinant of the Jacobian of $h^{-1}$.  
    
   \corr{From the discussion above, we get} that $\vert J_{h^{-1}} \vert > 0$ for all $x$.  In particular, since the determinant of a matrix is the product of its eigenvalues, we have that
    \[
        \corr{L_2}^{-\corr{n}} \leq \vert J_{h^{-1}}(x) \vert = \prod_{j=1}^n \lambda_{j} (J_{h^{-1}}(x)) \leq \corr{L_1}^{-n}.
    \]
    Finally, since $\mu$ itself adheres to (A1), this implies that
    \[
        \frac{b}{\corr{L_2}^n} \leq f_{\mu}(x) \vert J_{h^{-1}}(x) \vert \leq \frac{B}{\corr{L_1}^n}.
    \]
    So $(A1)$ holds for $\widetilde{\mu}$ if there are constants \corr{$L_1, L_2 > 0$ such that
    \[
        L_1 I \preceq J_{h}(x) \preceq L_2 I.
    \]}

    \item \textbf{Ensuring that $\mu_i$ satisfy $(A2)$:}  From \cite[Corollary 4.2.10]{hiriart1996}, we can ensure that $(A2)$ is satisfied if $(A3)$ is satisfied, which is proved below.

    \item \textbf{Ensuring that $\mu_i$ satisfy $(A3)$:}  First, notice that by compatibility of $h$, we have that $T_{\sigma}^{h_\sharp \mu} = h \circ T_{\sigma}^{\mu}$; thus, a direct corollary of \cite[Theorem 24]{khurana2022supervised} gives that
    \begin{align*}
        (m \corr{L_1}) I \preceq J_{T_\sigma^{h_\sharp \mu } }(x) \preceq (\corr{L_2}  L) I
    \end{align*}
    for all $x$, where $m$ and $L$ come from assuming $\sigma$ and $\mu$ satisfy (A3) whilst \corr{$L_1$ and $L_2$ come from above.}  So $(A3)$ holds for $\sigma$ and $\widetilde{\mu}$.\qedhere
\end{itemize}
\end{proof}

The result above essentially states that the entropic estimator works if every $h \in \mathcal{H}$ is (exactly) compatible and is uniformly positive definite.

\subsection{Non-Compact Case Proofs and Results}

Here we prove the results of \Cref{SEC:CondNonCpt} which provide conditions on $\sigma$, $\mu$, and $\mathcal{H}$ which guarantee that $\mu_i\sim\mathcal{H}_\sharp \mu$ satisfy the conditions of the theorems from \Cref{sec:non_compact_support}.

\begin{proof}[Proof of \Cref{THM:baryCondNonCompact}]\label{PF:baryCondNonCompact}
Assume that $\widetilde{\mu}$ is the truncated measure approximating $h_\sharp \mu$ for $h \in \mathcal{H}$.  Given the assumptions of \Cref{LEM:extLEMM2}, the truncated measure $\widetilde{\mu}$ is compactly supported, upper and lower bounded, and absolutely continuous.  If we can ensure that the truncated measure $\widetilde{\mu}$ also has uniformly convex support, we will fulfill the conditions of Caffarelli's regularity theorem, which guarantees that the optimal transport map is Lipschitz continuous.


\begin{itemize}
    \item \textbf{Decay rate condition:}  Assuming that $\mu$ has the necessary decay rate $f_\mu(x) \leq C < \infty$ and $0 < c \leq f_\mu(x)$ on a large enough ball where the decay rate is active, we need that $h_\sharp \mu = \overline{\mu}$ also has the same decay rate up to a constant.  For what follows, we must show that $h \in \mathcal{H}$ has an inverse $h^{-1}$.  \corr{Indeed, because of the bi-Lipschitz assumption (iii) of \Cref{THM:baryCondNonCompact}, we know that
    \begin{align*}
        L_1 I \preceq J_{h}(x) \preceq L_2 I.
    \end{align*}
    Because $J_h(x)$ is invertible for all $x$, we can Hadamard's inverse function theorem to conclude that $h^{-1}$ exists.  Moreover, this implies that
    \begin{align*}
        L_2^{-1} I \preceq J_{h^{-1}}(x) \preceq L_1^{-1} I.
    \end{align*}
    Since $\mathcal{H}$ satisfies \Cref{AS:epsCompatAssump} (iii) (i.e.
    \begin{align*}
        a \Vert x \Vert \leq \Vert h( x ) \Vert \leq A \Vert x \Vert 
    \end{align*}
    for some $a, A > 0$), then we know that
    \[
        A^{-1} \Vert x \Vert \leq \Vert h^{-1}(x) \Vert \leq a^{-1} \Vert x \Vert,\] or equivalently,
    \[\frac{A^{-1}}{\Vert h^{-1}(x) \Vert} \leq \frac{1}{\Vert x \Vert} \leq \frac{a^{-1}}{\Vert h^{-1}(x) \Vert}. 
    \]}
    Thus, for $\Vert x \Vert \geq A R$ (so that $\Vert h^{-1}(x) \Vert \geq R$) and the bounds above, we find that
    \begin{align*}
        f_{\overline{\mu}}(x) &= f_{\mu}( h^{-1}(x)) \underbrace{\vert J_{h^{-1}}(x) \vert}_{\leq L_1^{-n}} \\
        &\leq \bigg( \frac{\eta}{C_{n, p, \Omega, M_p}} \bigg)^{\frac{6p+16n}{p}} \frac{1}{ C' \Vert h^{-1} (x) \Vert^{n+2} } L_1^{-n} \\
        &\leq \bigg( \frac{\eta}{C_{n, p, \Omega, M_p}} \bigg)^{\frac{6p+16n}{p}} \frac{1}{ C' \Vert x \Vert^{n+2} } L_1^{-n} A^{n+2}.
    \end{align*}
    The constants $L_1$ and $A$ can be absorbed into the other decay rate constants; thus, \Cref{AS:epsCompatAssump} \corr{(iii) and our bi-Lipschitz assumption (iii)} gives us the decay rate we want.  Noting that the form of the density $f_{\overline{\mu}}$ also implies that $c L_1^{-n} \leq f_{\overline{\mu}}(x)$ on some large enough ball.  In particular, we get that the truncated measure $\widetilde{\mu}$ has a density $0 < b \leq f_{\widetilde{\mu}}(x) \leq B < \infty$ from \Cref{LEM:extLEMM2}.
    
    \item \textbf{Uniformly convex support:}  If $\mu$ is supported on all of $\R^n$, we would want $h \in \mathcal{H}$ such that $\overline{\mu} = h_\sharp \mu$ is also supported on all of $\R^n$.  Recall that the resulting density of $\overline{\mu}$ is given by
    \begin{align*}
        f_{\overline{\mu}}(x) &= f_{\mu}( h^{-1}(x)) \underbrace{\vert J_{h^{-1}}(x) \vert}_{\leq L_1^{-n}}
    \end{align*}
    Note that $\overline{\mu}$ is supported on all of $\R^n$ if $\Vert h^{-1}(x) \Vert \to \infty$ as $\Vert x \Vert \to \infty$.  Indeed, if we assume \Cref{AS:epsCompatAssump} \corr{(iv)}, then $A^{-1} \Vert x \Vert \leq \Vert h^{-1}(x) \Vert$, which implies that $\overline{\mu}$ is supported on all of $\R^n$.  This would imply that the truncated measure $\widetilde{\mu}$ will be supported on a ball of some radius.  This implies that the support of $\widetilde{\mu}$ is uniformly convex and compact.
\end{itemize}

From the decay rate condition and the uniformly convex support condition, we get that the truncated measure $\widetilde{\mu}$ will satisfy the assumptions of Caffarelli's regularity theorem.  This implies that $T_\sigma^{\widetilde{\mu}}$ will be a $C^2$ and Lipschitz function (since $T_\sigma^{\widetilde{\mu}}$ pushes forward a compact support to a compact support).  The other assumptions of the theorem are trivially satisfied.
\end{proof}

\begin{proof}[Proof of \Cref{THM:sinkCondNonCompact}]\label{PF:sinkCondNonCompact}
From the proof of \Cref{THM:baryCondNonCompact} above, we easily see that if \Cref{AS:epsCompatAssump} is fulfilled and $\mu$ fulfills the conditions of \Cref{LEM:extLEMM2} and is supported on all of $\R^n$, then $T_{\sigma}^{\widetilde{\mu}}$ will be Lipschitz.  We need, however, that $\widetilde{\mu}$ also satisfies $(A1)$-$(A3)$ from \ref{A1}.  We get $(A1)$ for free since the density $f_{\widetilde{\mu}}$ is lower bounded from the proof of \Cref{LEM:extLEMM2}.  We also get $(A2)$ since $T_{\sigma}^{\widetilde{\mu}}$ is differentiable from Caffarelli's regularity theorem \cite{caffarelli92,caffarelli92b,caffarelli96} and if (A3) is satisfied, which comes from \cite[Corollary 4.2.10]{hiriart1996}.

Now we only need to ensure that $(A3)$ holds.  Indeed, since Caffarelli's regularity theorem holds, we know that the potential $\phi$ such that $T_{\sigma}^{\widetilde{\mu}} = \nabla \phi$ is strictly convex, which implies that $\nabla^2 \phi(x)$ is positive definite.  Moreover, the minimum eigenvalue of $\nabla^2 \phi(x)$ is a continuous function of $x$.  Since $x \in \supp(\sigma)$, which is compact, we know that $0 < \lambda_{\min}(\sigma) = \min_{x \in \supp(\sigma)} \lambda_{\min}(\nabla^2 \phi(x))$, which implies that $J_{T_\sigma^{\widetilde{\mu}}}(x) \succeq \lambda_{\min}(\sigma) I$.  This guarantees that $(A3)$ is satisfied for $\sigma$ and $\widetilde{\mu}$.
\end{proof}

\end{document}